\documentclass[11pt]{article}

\usepackage[margin=1in]{geometry}
\usepackage{amsmath,amssymb,amsfonts,amsthm,mathtools,mathrsfs}
\usepackage{bm}
\usepackage{enumitem}
\usepackage{xcolor}
\usepackage{float}
\usepackage{graphicx}
\usepackage{booktabs}
\usepackage{tikz}
\usepackage[caption=false]{subfig}
\usepackage{natbib}

\usetikzlibrary{arrows.meta}
\usepackage[colorlinks=true,linkcolor=blue,citecolor=blue,urlcolor=blue]{hyperref}
\providecommand{\keywords}[1]{%
  \par\smallskip\noindent\textbf{Keywords: }#1\par}

\numberwithin{equation}{section}

\theoremstyle{plain}
\newtheorem{theorem}{Theorem}[section]
\newtheorem{proposition}[theorem]{Proposition}
\newtheorem{lemma}[theorem]{Lemma}
\newtheorem{corollary}[theorem]{Corollary}

\theoremstyle{definition}
\newtheorem{assumption}[theorem]{Assumption}

\theoremstyle{remark}
\newtheorem{remark}[theorem]{Remark}

\newcommand{\R}{\mathbb R}

\newcommand{\E}{\mathbb E}

\newcommand{\calD}{\mathcal D}
\newcommand{\calL}{\mathcal L}

\newcommand{\calK}{\mathcal K}

\newcommand{\ip}[2]{\left\langle #1,#2\right\rangle}
\newcommand{\Var}{\operatorname{Var}}

\newcommand{\cP}{\mathcal P}
\newcommand{\cE}{\mathcal E}
\newcommand{\cI}{\mathcal I}
\newcommand{\cA}{\mathcal A}
\newcommand{\cS}{\mathcal S}
\newcommand{\calG}{\mathcal G}
\newcommand{\cW}{\mathcal W}
\newcommand{\norm}[1]{\left\lVert #1\right\rVert}
\newcommand{\abs}[1]{\left\lvert #1\right\rvert}
\allowdisplaybreaks

\title{Accelerating Non-Smooth and Heavy-Tailed Sampling}
\author{
Pervez Ali\thanks{Department of Mathematics, Florida State University, Tallahassee, Florida, United States of America; \texttt{pa22g@fsu.edu}}
\and
Xiaoyu Wang\thanks{Fintech Thrust, Hong Kong University of Science and Technology (Guangzhou), Guangzhou, Guangdong Province, People's Republic of China;
\texttt{xiaoyuwang@hkust-gz.edu.cn}}
\and
Yingli Wang\thanks{School of Mathematical Sciences, Fudan University, Shanghai, People's Republic of China; \texttt{yingliwang@fudan.edu.cn}}
\and
Lingjiong Zhu\thanks{Department of Mathematics, Florida State University, Tallahassee, Florida, United States of America; \texttt{zhu@math.fsu.edu}}}
\date{\today}

\begin{document}
\maketitle

\begin{abstract}
Anchored Langevin dynamics (ALD) is useful for non-smooth sampling where the density 
of the target distribution is
possibly non-differentiable and heavy-tailed; reflected anchored Langevin dynamics (RALD) can sample possibly non-differentiable target density on a constrained domain. In this paper, we propose and study non-reversible anchored Langevin dynamics (NALD)
for sampling possibly non-differentiable and heavy-tailed target density in the Euclidean space
and the non-reversible reflected anchored Langevin dynamics (NRALD) for sampling possibly non-differentiable target density
in the constrained space.  Our construction adds a circulation drift generated by a
possibly state-dependent divergence-free skew-symmetric matrix field
and a stream potential.  It preserves the target distribution without requiring
derivatives of target density, admits a random-time-change representation, and applies both on
the whole Euclidean space and on bounded domains with normal reflection. 
By breaking reversibility, we show
that NALD and NRALD can converge to their target distributions
faster than their reversible counterparts
via finite-time non-asymptotic convergence analysis,  
a large deviations analysis
and asymptotic variance reduction.
Numerical experiments demonstrate
the efficiency of the proposed algorithms.
\end{abstract}


\section{Introduction}

Let $E$ be either $\R^d$ or the closure of a bounded domain $K\subset\R^d$, and
consider the problem of sampling
\begin{equation}\label{eq:target-intro}
    \pi(dx)=Z^{-1}e^{-U(x)}\,dx,
    \qquad
    Z=\int_E e^{-U(x)}\,dx<\infty.
\end{equation}
Such problems arise in many applications 
such as Bayesian learning problems in machine learning \citep{gelman1995bayesian,stuart2010inverse,andrieu2003introduction,teh2016consistency,DistMCMC19,GIWZ2024}.

Langevin algorithms are popular
\textit{Markov chain Monte Carlo} (MCMC) methods
for solving the sampling problem 
\eqref{eq:target-intro}.
When $E=\mathbb{R}^{d}$ in \eqref{eq:target-intro}, 
the most classical Langevin algorithm
for sampling \eqref{eq:target-intro}, 
the \textit{Unajudsted Langevin Algorithm} (ULA), 
is based on the discretization
of the \textit{overdamped Langevin dynamics}
\citep{Dalalyan,DM2017,DK2017,Raginsky}:
\begin{equation}\label{eq:overdamped-2}
dX_{t}=-\nabla U(X_{t})dt+\sqrt{2}dW_{t},
\end{equation}
where $W_{t}$ is a standard $d$-dimensional Brownian motion,
and it is known that under mild conditions,
the overdamped Langevin dynamics \eqref{eq:overdamped-2} admits the target distribution
\eqref{eq:target-intro} with $E=\mathbb{R}^{d}$ as the unique invariant distribution \citep{chiang1987diffusion,stroock-langevin-spectrum}.

The overdamped Langevin dynamics \eqref{eq:overdamped-2} requires evaluations of $\nabla U$ and may therefore
be unavailable when $U$ is non-smooth.  Several approaches replace the non-smooth target
by a differentiable surrogate.  These include smoothing through parametrized exponential
families \citep{Zhou14}, \textit{Moreau--Yoshida envelope} (MYE) and proximal constructions
\citep{Durmus18,Goldman21,Mou19,Salim20}. Other envelopes such as the forward-backward envelope to approximate a non-smooth potential with a smooth one can also be used \citep{Eftekhari2023}.
When the potential function $U$
is non-differentiable but convex,
\cite{Durmus19,habring2024subgradient} provide an alternative; but they are restrictive to the convex setting. 
By approximating $U$ by convolving it with a Gaussian kernel, 
Gaussian smoothing method 
\citep{nesterov2017random,Chatterji20} provide another alternative. 
Such methods have produced useful algorithms and
nonasymptotic guarantees, particularly under convexity, but they alter the drift through a
regularization or approximation of the original potential.
For non-differentiable target distributions, zeroth-order Langevin algorithms can also be employed. These methods approximate first-order information based on the evaluations of the target function $U$ using finite-differences \citep{roy2022stochastic,dwivedi2019log}. Such an approach can be beneficial in some settings when the gradients do not exist or when they are hard to compute. However, zeroth-order methods are typically slower than first-order methods, as their gradient estimates tend to be noisier. 
\textit{Piecewise-deterministic Markov} processes (PDMPs) such as the bouncy particle and zig-zag samplers do not suffer from bias due to discretization; they can be applied to target distributions that are differentiable almost everywhere \citep{deligiannidis2019exponential,durmus2020geometric,bierkens2019ergodicity}, and in practice, may be preferable to MYE-based methods—particularly when the MYE is difficult to compute or not well-defined 
\citep{Goldman21}. However, PDMPs may encounter computational difficulties in high dimensions; event-time simulations can be demanding, and their performance can be sensitive to the availability of tight event-rate bounds \citep{Goldman21}.

\textit{Anchored Langevin dynamics} (ALD) \citep{gurbuzbalaban2025anchored} takes a different,
probabilistic route.  It introduces a differentiable reference potential $U_0$ and the
state-dependent speed
\begin{equation}\label{eq:a-intro}
    a(x):=e^{U(x)-U_0(x)},
\end{equation}
such that 
the \textit{anchored Langevin dynamics} (ALD) satisfies the SDE \citep{gurbuzbalaban2025anchored}:
\begin{equation}\label{eqn:anchored}
dX_{t}=-a(X_{t})\nabla U_{0}(X_{t})dt+\sqrt{2a(X_{t})}dW_{t},
\end{equation}
and it was proved in \cite{gurbuzbalaban2025anchored}
that \eqref{eqn:anchored} preserves
the Gibbs distribution \eqref{eq:target-intro} as an invariant distribution 
when $E=\mathbb{R}^{d}$.
The infinitesimal generator of \eqref{eqn:anchored} is given by
\begin{equation}\label{eq:reversible-generator-intro}
    \cS f
    :=a\left(\Delta f-\nabla U_0\cdot\nabla f\right)
    =\pi^{-1}\nabla\cdot\left(\pi a\nabla f\right).
\end{equation}
Thus, the target potential $U$ enters through function evaluations in $a$, whereas the
drift uses the gradient of the reference potential.  The divergence-form identity in
\eqref{eq:reversible-generator-intro} shows that $\cS$ is self-adjoint in $L^2(\pi)$.
The same construction can be combined with normal reflection to sample a target
supported on a constrained domain; this reflected anchored diffusion has recently been
studied in \cite{RALD}.

The anchored Langevin dynamics in \eqref{eqn:anchored} only requires the differentiability of $U_{0}$, not
the target potential $U$. 
As a result, \eqref{eqn:anchored} provides
an alternative method for non-smooth sampling where the target density is not necessarily differentiable. 
Moreover, it is known that
the overdamped Langevin dynamics \eqref{eq:overdamped-2} can be inefficient
and slow for sampling \textit{heavy-tailed} distributions \citep{RT96}. 
Sampling from target distributions with heavy tails—where the negative log-density \( U(x) \) grows sublinearly—presents unique challenges for standard sampling algorithms and the literature is quite limited. 
To address this, several algorithmic modifications have been proposed \citep{Kamatani2017,belomestny2021fourier,bell2024adaptive}; however, these approaches rely on Metropolis steps, which can be expensive for some applications involving high dimensionality and large datasets \citep{welling2011bayesian}.
\cite{he-mean-square-analysis} develops zeroth- and first-order Langevin algorithms for heavy-tailed distributions whose potentials satisfy weighted Poincar\'e inequalities, including Student-\(t\) distributions. In related work, \cite{he-transformed} characterize the class of heavy-tailed densities for which polynomial-order complexity guarantees can be obtained when the Unadjusted Langevin Algorithm is applied to suitably transformed versions of the target. 
The anchored Langevin dynamics
and the discretized algorithms
in \cite{gurbuzbalaban2025anchored}
provided a complementary approach
to sampling heavy-tailed distributions.
It was shown in \cite{gurbuzbalaban2025anchored} that anchored Langevin dynamics can converge
to a heavy-tailed target distribution exponentially fast in time; hence, it provides a unified approach for sampling both non-smooth and heavy-tailed distributions. 

Constrained sampling has gained a lot of interest in recent years.  
\textit{Projected Langevin Monte Carlo} (PLMC) algorithms were
analyzed first for log-concave targets and subsequently in stochastic and nonconvex
settings \citep{bubeck2015finite,bubeck2018sampling,Lamperski2021,zheng2022constrained}, which
are based on the discretization of
the \textit{reflected Langevin dynamics} (RLD):
\begin{equation}\label{eqn:RLD}
dX_{t}=-\nabla U(X_{t})dt+\sqrt{2}dW_{t}-n(X_{t})dL_t,
\end{equation}
where $W_{t}$ is a standard $d$-dimensional Brownian motion, the term $n(X_{t})$ is the outward unit normal vector when $X_t\in\partial K$ and $-n(X_{t})dL_t$ reflects the process on the boundary to ensure $X_{t}\in K$
for every $t$ given that $X_{0}\in K$. 
In particular, $-\int_{0}^{t}n(X_s)dL_s$ is a bounded variation reflection process,
where $L_t$ is nondecreasing and finite for every $t$, and $dL_t$ is supported on
$\{t:X_t\in\partial K\}$.
It is proved in \cite{bubeck2015finite} that
the RLD \eqref{eqn:RLD} converges to the target distribution \eqref{eq:target-intro}
with $E=K$ as $t\rightarrow\infty$.

Other approaches for constrained sampling include proximal methods \citep{Brosse,Salim20,2026IslamZhu}, mirrored Langevin
dynamics \citep{hsieh2018mirrored,Chewi2020,Zhang2020,TaoMirror2021,Ahn2021}, and
penalty-based or replica-exchange constructions \citep{GHZ2022,constraint_replica}.
Nonreversible constrained diffusions with skew reflection were introduced in
\cite{DFTWZ2025}.  More recently, \cite{wang2026constrained} showed that a
state-dependent skew field satisfying $J(x)n(x)=0$ on the boundary restores ordinary
normal reflection and yields both large-deviation and asymptotic-variance acceleration.
To sample \eqref{eq:target-intro}
for $E=K$ with a possibly non-differentiable $U(x)$, 
\cite{RALD} proposed and studied \textit{reflected anchored Langevin dynamics} (RALD):
\begin{equation}\label{eqn:RALD}
dX_{t}=-a(X_{t})\nabla U_{0}(X_{t})dt+\sqrt{2a(X_{t})}dW_{t}-n(X_{t})dL_t,
\end{equation}
where $U_{0}:\mathbb{R}^{d}\rightarrow\mathbb{R}$
plays the role as an anchored reference potential, $W_{t}$ is a standard $d$-dimensional Brownian motion, the term $n(X_{t})$ is the outward unit normal vector when $X_t\in\partial K$ and $-n(X_{t})dL_t$ reflects the process on the boundary to ensure $X_{t}\in K$
for every $t$ given that $X_{0}\in K$. 
In particular, $-\int_{0}^{t}n(X_s)dL_s$ is a bounded variation reflection process,
where $L_t$ is nondecreasing and finite for every $t$, and $dL_t$ is supported on
$\{t:X_t\in\partial K\}$.

In the literature, it is known
that in the setting of unconstrained sampling, 
by \textit{breaking reversibility}, 
a \textit{non-reversible} Langevin dynamics, 
where an \textit{skew-symmetric} matrix is added,
can converge to the target distribution faster on the Euclidean space $\mathbb{R}^{d}$. 
This phenomenon is well understood through non-asymptotic
convergence analysis; see e.g. \cite{HHS93,HHS05,reybellet2015irreversible,HWGGZ20,GGZ2}.
The improvement has also been studied
through spectral and nonasymptotic convergence \citep{Lelievre-optdrift,GGZ2,HWGGZ20},
the level-two large-deviation rate function \citep{reybellet2015irreversible}, and the
asymptotic variance of ergodic averages \citep{duncan2016variance}.  Weighted-topology
large deviations for unbounded diffusions provide the coercive framework needed beyond
compact state spaces \citep{ferre2020large,Yao2025LDP}.  
Recently, this idea also inspired the study of non-reversible Langevin algorithms
in constrained sampling. 
\textit{Non-reversible skew-reflected Langevin dynamics} (SRNLD) was
proposed and studied in \cite{DFTWZ2025} where an anti-symmetric matrix $J$ is added to the
dynamics, that can accelerate the convergence 
of constrained sampling.
However, how to choose the anti-symmetric matrix $J$ in practice is a challenging problem. 
Very recently,  
\cite{wang2026constrained} showed acceleration for constrained sampling through the lens of large deviations and asymptotic variance for SRNLD under the
additional constraint that $J(x)n(x)=0$,
where $n(x)$ is the outward unit normal vector. For the examples of ball constraint, and more generally the constrained domains characterized by the sub-level sets that include smoothed $\ell_{p}$ balls as special cases, the explicit construction of $J(x)$ was provided in \cite{wang2026constrained}. Indeed, under this additional condition, 
skew-reflection becomes a normal reflection. Extensive numerical experiments using both synthetic and real data in \cite{wang2026constrained} demonstrated superior performance for the choice of $J$ that satisfies this additional condition.

In \cite{gurbuzbalaban2025anchored}, it was proved that anchored Langevin dynamics is reversible. 
In addition, \cite{RALD} proved that reflected anchored Langevin dynamics is also reversible.
Therefore, it motivates us to ask the question whether
by breaking reversibility, one can accelerate the convergence
of anchored and reflected anchored Langevin dynamics for sampling
not necessarily differentiable target.
In this paper, we will give positive answer to this question by proposing
and studying \textit{non-reversible anchored Langevin dynamics} (NALD) for unconstrained
sampling and \textit{non-reversible reflected anchored Langevin dynamics} (NRALD)
for constrained sampling. 

Our main contributions are as follows.
\begin{enumerate}[leftmargin=2em]
    \item We introduce non-reversible anchored Langevin dynamics (NALD) on $\R^d$ and its
    normally reflected counterpart, non-reversible reflected anchored Langevin dynamics (NRALD), on bounded domains.  
    Both preserve the given target distributions as the invariant distributions on $\R^d$ and the bounded domain respectively (Theorem~\ref{thm:invariance})
    and arise as random time changes of nonreversible diffusions associated with the
    reference potential $U_0$ (Proposition~\ref{prop:time-change}).  The current construction includes divergence-free
    state-dependent skew tensors as well as constant matrices.
    On a constrained domain, we identify two mechanisms for retaining ordinary
    normal reflection: a stream potential with constant boundary trace, or a
    state-dependent tensor satisfying $J(x)n(x)=0$.  We explain the two-dimensional
    degeneracy of the latter condition and give a three-dimensional cross-product
    construction (Lemma~\ref{lem:stream-boundary} and
    Remark~\ref{rem:dimension-Jn}).  The stream mechanism remains nontrivial when
    $U_0$ is constant (Remark~\ref{rem:constant-anchor}), while
    Proposition~\ref{prop:state-dependent-current} places both mechanisms in the same
    divergence-free current framework.
    \item We obtain finite-time non-asymptotic convergence guarantees in total variation (TV) distance and show
    that by breaking reversibility, NALD and NRALD can converge to the target distributions
    faster than their reversible counterparts (Theorem~\ref{thm:TV:unconstrained} and Theorem~\ref{thm:TV}). 
    \item In order to better quantify the acceleration, we establish a level-two large
    deviation principle under standard compact-state assumptions
    (Theorem~\ref{thm:rate-decomposition}) or under weighted Witten--Lyapunov
    assumptions on $\R^d$ (Proposition~\ref{prop:unconstrained-ldp-sufficient} and
    Theorem~\ref{thm:unconstrained-ldp}).  If $d\nu=e^v d\pi$, then the rate function has the
    decomposition (Theorem~\ref{thm:rate-decomposition} and Theorem~\ref{thm:unconstrained-ldp}):
    \begin{equation}\label{eq:rate-preview}
        \cI_\alpha(\nu)
        =\frac14\int_E a\abs{\nabla v}^2d\nu
        +\frac{\alpha^2}{4}
        \norm{\cA v}_{H^{-1}_a(\nu)}^2.
    \end{equation}
    Consequently $\cI_\alpha\geq \cI_0$, with a transparent strictness criterion. Hence, the non-reversible dynamics
    is more concentrated around the target distribution which justifies the acceleration.
    \item In addition, we study asymptotic variance to quantify the acceleration. Under the assumption of a weighted Poincar\'e inequality, we prove that for every centered
    observable $g\in L^2(\pi)$ (Theorem~\ref{thm:variance-reduction}):
   $\sigma^2_{g,\alpha}
        \leq \sigma^2_{g,0}$.
    Equality holds precisely when the reversible Poisson solution is not moved by the
    circulation, under the operator assumptions stated below.
    \item For the $q$-parametrized heavy-tailed family of
    \cite[Theorem~3]{gurbuzbalaban2025anchored}, the same Foster--Lyapunov estimate yields
    a weighted Poincar\'e inequality when $r\geq0$.  When $r>0$, tangency of the
    circulation yields the principal coercive Lyapunov estimate for the large deviation principle (LDP).  We verify
    the remaining Witten--Lyapunov conditions for a Student-$t$ specialization (Proposition~\ref{prop:q-heavy-tail}). For
    this same superquadratic model, we prove the strong sector condition by separating
    radial and angular derivatives, thereby obtaining both LDP (Corollary~\ref{cor:student-ldp}) and asymptotic-variance
    acceleration (Corollary~\ref{cor:student-superquadratic-variance}).  We also give an explicitly solvable quadratic-clock example (Proposition~\ref{prop:student-variance}).
    \item Our experiments show that breaking reversibility accelerates the anchored
Langevin algorithms which confirms the theory (Section~\ref{sec:numerical}). NALD~\eqref{eq:algorithm-unconstrained}
with the skew-symmetric matrix $J_a$ ~\eqref{eq:J_a} reaches the noise floor on the heavy-tailed target with the MCP regularizer~\eqref{eq:student-mcp} in about $1.8$ times
fewer iterates, NRALD~\eqref{eq:algorithm-constrained} with the state-dependent skew-symmetric matrix $J_s$
in~\eqref{eq:Js} speeds up convergence to the constrained Laplace
targets by a factor of up to $1.85$. In constrained Bayesian logistic
regression experiment, NRALD ~\eqref{eq:algorithm-constrained} reaches a given accuracy faster on both synthetic and real
data.
\end{enumerate}

The rest of the paper is organized as follows.  Section~\ref{sec:model} introduces the
dynamics and proves invariance and the time-change representation.
Section~\ref{sec:non:asymptotic} presents finite-time non-asymptotic convergence analysis.
Section~\ref{sec:ldp} studies empirical-measure large deviations.  The asymptotic
variance comparison is developed in Section~\ref{sec:variance}.  Explicit unconstrained
and reflected constructions are given in Section~\ref{sec:examples}. Numerical experiments are presented in Section~\ref{sec:numerical}.
Finally, we conclude in Section~\ref{sec:discussion}.

\section{Main Results}\label{sec:model}

We aim to design non-reversible anchored Langevin dynamics 
in both unconstrained and constrained
settings that can accelerate
the convergence
to the target distributions.
A first attempt would replace $\nabla U_0$ by $(I_{d}+\alpha J)\nabla U_0$, where
$J$ is a constant skew-symmetric matrix.  This is a valid and useful special case, but
it becomes reversible when $U_0$ is constant in the constrained setting.  The latter is precisely the choice that
produces the simplest zeroth-order anchored sampler on a compact domain.  We therefore
use a more general circulation construction.  Let $J(x)$ be a possibly state-dependent
skew-symmetric matrix field satisfying $\nabla\cdot J=0$, and let $\psi$ be a stream
potential.  Define
\begin{equation}\label{eq:circulation-intro}
    c(x):=e^{U(x)}J(x)\nabla\psi(x).
\end{equation}
Since $\pi c=Z^{-1}J\nabla\psi$, we have
\begin{align}\label{eq:weighted-div-free-intro}
    \nabla\cdot(\pi c)
    &=Z^{-1}\sum_{i=1}^{d}\sum_{j=1}^{d}
      \left[(\partial_iJ_{ij})(\partial_j\psi)
             +J_{ij}\partial_{ij}\psi\right]=0.
\end{align}
The first term vanishes because $\nabla\cdot J=0$, while the second vanishes because
$J$ is skew-symmetric and $\nabla^2\psi$ is symmetric.
On a bounded domain we additionally choose $\psi$ so that $c$ is tangential at the
boundary.  The resulting infinitesimal generator is
\begin{equation}\label{eq:generator-intro}
    \calL_\alpha=\cS+\alpha\cA,
    \qquad
    \cS f:=a\left(\Delta f-\nabla U_0\cdot\nabla f\right)
           =\pi^{-1}\nabla\cdot(\pi a\nabla f),
    \qquad
    \cA f:=c\cdot\nabla f,
\end{equation}
where $\cS$ is symmetric and $\cA$ is anti-symmetric on their natural common core in
$L^2(\pi)$.  This common form structure is the link between the unconstrained and
reflected parts of the paper.  The stronger skew-adjoint statement needed for the
variance comparison is established below for the sandwiched operator \eqref{eq:G-def}.

Throughout the paper, $d\geq2$ and $\alpha\in\R$ controls the strength of the
nonreversible perturbation.  Let $J\in C^1(E;\R^{d\times d})$ be a possibly
state-dependent matrix field.  For clarity, we interpret its divergence columnwise:
\begin{equation}\label{eq:tensor-divergence}
    (\nabla\cdot J)_j:=\sum_{i=1}^d\partial_iJ_{ij},
    \qquad j=1,\ldots,d.
\end{equation}
Our general assumptions on $J$ are
\begin{equation}\label{eq:state-dependent-J-assumptions}
    J(x)^\top=-J(x),
    \qquad
    \nabla\cdot J(x)=0.
\end{equation}
When $J(x)$ is independent of $x$, we simply write $J$ and refer to it as the
constant-matrix case.  This specialization is used in the unconstrained Student
analysis and in the explicit examples.

\subsection{Unconstrained dynamics}

Let $E=\R^d$.  We assume that $U$ is locally Lipschitz, $U_0\in C^2(\R^d)$,
$\psi\in C^2(\R^d)$, and that $J$ satisfies
\eqref{eq:state-dependent-J-assumptions}.  Set
\begin{equation}\label{eq:a-c-def}
    a(x):=e^{(U-U_0)(x)},
    \qquad
    c(x):=e^{U(x)} J(x)\nabla\psi(x),
    \qquad
    b_\alpha(x):=-a(x)\nabla U_0(x)+\alpha c(x).
\end{equation}
We propose the \textit{non-reversible anchored Langevin dynamics} (NALD) as:
\begin{equation}\label{eq:nald}
    dX_t=b_\alpha(X_t)\,dt+\sqrt{2a(X_t)}\,dW_t.
\end{equation}

\begin{assumption}[Local regularity and well-posedness]
\label{ass:unconstrained-local}
The functions $a$ and $b_\alpha$ are locally Lipschitz, $a(x)>0$ for every
$x\in\R^d$, and the solution of \eqref{eq:nald} is nonexplosive.  Whenever
irreducibility is invoked, the diffusion is also assumed to be topologically
irreducible.
\end{assumption}
Notice that \eqref{eq:nald} does not require the everywhere existence of $\nabla U$.
The special choice
\begin{equation}\label{eq:canonical-psi}
    \psi=e^{-U_0}
\end{equation}
gives $c(x)=-a(x)J(x)\nabla U_0(x)$ and hence \eqref{eq:nald} becomes:
\begin{equation}\label{eq:canonical-anld}
    dX_t=-a(X_t)\left(I_{d}+\alpha J(X_t)\right)\nabla U_0(X_t)\,dt
          +\sqrt{2a(X_t)}\,dW_t,
\end{equation}
which is the most direct anti-symmetric perturbation of anchored Langevin dynamics.

\subsection{Divergence-free state-dependent skew fields}

Here $E$ may be either $\R^d$ or $\overline K$, as appropriate.  The assumptions in
\eqref{eq:state-dependent-J-assumptions} ensure that the general matrix field $J(x)$
produces a divergence-free weighted current.  For a skew-symmetric matrix field,
vanishing column divergence is equivalent, up to sign, to vanishing row divergence.

\begin{proposition}[State-dependent stream currents]
\label{prop:state-dependent-current}
Let $J\in C^1(E;\R^{d\times d})$ satisfy
\eqref{eq:state-dependent-J-assumptions}.  For $\psi\in C^2(E)$, the current
$c$ in \eqref{eq:a-c-def} satisfies $\nabla\cdot(\pi c)=0$.
Consequently, whenever the corresponding diffusion is well posed and, in the reflected
case, $c\cdot n=0$, all the invariance and form-anti-symmetry conclusions of
Theorem~\ref{thm:invariance} hold for the general state-dependent model.

In particular, the canonical choice $\psi=e^{-U_0}$ gives
\begin{equation}\label{eq:state-dependent-canonical}
    dX_t=-a(X_t)\left(I_{d}+\alpha J(X_t)\right)\nabla U_0(X_t)\,dt
          +\sqrt{2a(X_t)}\,dW_t.
\end{equation}
On $\overline K$, the ordinary normal reflection term $-n(X_t)dL_t$ is added.
\end{proposition}

\begin{proof}
Since $\pi c=Z^{-1}J\nabla\psi$, we have
\begin{align*}
    \nabla\cdot(\pi c)
       =Z^{-1}\sum_{i=1}^d\sum_{j=1}^d
      (\partial_iJ_{ij})(\partial_j\psi)
             +Z^{-1}\sum_{i=1}^d\sum_{j=1}^d
     J_{ij}\partial_{ij}\psi=0.
\end{align*}
The first sum vanishes by \eqref{eq:tensor-divergence}--
\eqref{eq:state-dependent-J-assumptions}, and the second is zero because a
skew-symmetric matrix is orthogonal in the Frobenius inner product to the symmetric
Hessian.  The remaining conclusions follow by the integration-by-parts argument of
Theorem~\ref{thm:invariance}.  Finally,
\[
e^UJ\nabla\left(e^{-U_0}\right)
=-e^Ue^{-U_0}J\nabla U_{0}=-aJ\nabla U_0,
\]
which proves
\eqref{eq:state-dependent-canonical}.
\end{proof}

\subsection{Normally reflected dynamics}

Let $K\subset\R^d$ be a bounded connected domain with sufficiently smooth boundary,
and write $n(x)$ for the outward unit normal at $x\in\partial K$.  We impose the
zero-flux condition:
\begin{equation}\label{eq:tangent-condition}
    c(x)\cdot n(x)=0,
    \qquad x\in\partial K.
\end{equation}
We propose the \textit{non-reversible reflected anchored Langevin dynamics} (NRALD) as:
\begin{equation}\label{eq:nrald}   dX_t=b_\alpha(X_t)\,dt+\sqrt{2a(X_t)}\,dW_t-n(X_t)dL_t,
    \qquad X_t\in\overline K,
\end{equation}
where $L_t$ is nondecreasing and $dL_t$ is supported on
$\{t:X_t\in\partial K\}$.  The infinitesimal generator has the
Neumann domain
\begin{equation}\label{eq:neumann-domain}
    \calD_N:=\left\{f\in C^2(\overline K):\nabla f\cdot n=0
    \text{ on }\partial K\right\}.
\end{equation}
We write
\begin{equation}\label{eq:positive-neumann-domain}
    \calD_{N,+}:=\left\{u\in\calD_N:u>0\ \text{on }\overline K\right\}
\end{equation}
for its strictly positive cone.
No oblique reflection is needed.

The boundary condition \eqref{eq:tangent-condition} can be enforced either through the
stream potential or directly through a state-dependent skew tensor.  The latter mechanism
is the normal-reflection condition used in \cite{wang2026constrained}.

\begin{lemma}[Two normal-reflection mechanisms]\label{lem:stream-boundary}
Let $J:\overline K\to\R^{d\times d}$ be pointwise skew-symmetric and set
$c=e^U J\nabla\psi$.  Either one of the following conditions is sufficient for
$c\cdot n=0$ on $\partial K$:
\begin{enumerate}[label=\textup{(\roman*)},leftmargin=2.5em]
    \item $\psi$ is constant on every connected component of $\partial K$;
    \item $J(x)n(x)=0$ for every $x\in\partial K$.
\end{enumerate}
\end{lemma}

\begin{proof}
Under (i), all tangential derivatives of $\psi$ vanish, so $\nabla\psi$ is parallel to
$n$ and $n^\top J\nabla\psi=0$.  Under (ii), skew-symmetry gives
$n^\top J\nabla\psi=-(Jn)\cdot\nabla\psi=0$.
\end{proof}

\begin{remark}[Dimension dependence of $Jn=0$]\label{rem:dimension-Jn}
The tensor boundary mechanism admits natural nontrivial constructions from dimension
three onward.  

(i) In $d=2$, every skew tensor has the form
\begin{equation*}
    J(x)=j(x)
    \begin{pmatrix}0&-1\\1&0\end{pmatrix}.
\end{equation*}
Condition $\nabla\cdot J=0$ forces $\nabla j=0$, while
$Jn=0$ at any boundary point then forces $j=0$.  Thus the simultaneous
conditions in Proposition~\ref{prop:state-dependent-current} and Lemma~\ref{lem:stream-boundary}(ii)
give only the trivial tensor in two dimensions.

(ii) In $d=3$, let
\begin{equation}\label{eq:cross-product-tensor}
    J(x)v=\nabla\phi(x)\times v.
\end{equation}
Then $J$ is skew-symmetric and
$\nabla\cdot J=\operatorname{curl}(\nabla\phi)=0$.  If $\phi$ is constant on
$\partial K$, then $\nabla\phi$ is normal to the boundary and hence
$Jn=0$.  On a ball, the choice $\phi(x)=\abs{x}^2/2$ gives
$J(x)v=x\times v$.  

(iii) On higher-dimensional balls, further examples follow by
embedding such $3\times3$ blocks.  More precisely, fix a triple of coordinate
indices $I=\{i_1,i_2,i_3\}$, let $P_I:\R^d\to\R^3$ be the corresponding
coordinate projection, and define
\begin{equation}\label{eq:higher-dimensional-ball-tensor}
    J_I(x)v
    :=P_I^\top\left((P_Ix)\times(P_Iv)\right).
\end{equation}
Then $J_I(x)$ is skew-symmetric, and a componentwise calculation with the
three-dimensional Levi--Civita tensor gives
\begin{equation*}
    \nabla\cdot J_I=0,
    \qquad
    J_I(x)x=P_I^\top\left((P_Ix)\times(P_Ix)\right)=0.
\end{equation*}
Consequently, on $\partial B_R(0)$, where $n(x)=x/R$, one has
$J_I(x)n(x)=0$.  Any finite linear combination of tensors of the form
\eqref{eq:higher-dimensional-ball-tensor} has the same three properties and
provides a nontrivial construction in every dimension $d\geq3$.
Notice, however, that this particular construction
annihilates a radial $\nabla U_0$; the constant-$J$ Student construction and the more
general stream choice therefore remain essential.
\end{remark}


\subsection{Sublevel-set constraints and normal reflection}

Let $\rho\in C^2(\R^d)$ and suppose
\begin{equation}\label{eq:sublevel-domain}
    K=\{x:\rho(x)<0\},
    \qquad
    \nabla\rho(x)\neq0\quad\text{on }\partial K.
\end{equation}
Choose $\psi=\rho$.  Since $\psi=0$ on $\partial K$, Lemma~\ref{lem:stream-boundary}
shows that
\begin{equation}\label{eq:sublevel-current}
    c(x)=e^{U(x)}J(x)\nabla\rho(x)
\end{equation}
is tangent to the boundary.  Thus \eqref{eq:nrald} preserves the target on $K$ with
ordinary normal reflection for any reference potential $U_0$ satisfying the regularity
assumptions.

For a ball $K=B(0,R)$, take
\begin{equation*}
    \rho(x)=\frac12\left(\abs{x}^2-R^2\right).
\end{equation*}
Then $c(x)=e^{U(x)}J(x)x$ is an angular drift.  In particular, for the constant anchor
$U_0\equiv0$, the sampler is
\begin{equation}\label{eq:zeroth-order-reflected}
    dX_t=\alpha e^{U(X_t)}J(X_t)X_t\,dt
          +\sqrt{2e^{U(X_t)}}\,dW_t-n(X_t)dL_t.
\end{equation}
It is genuinely nonreversible and uses $U$ only through function values.  
Equation \eqref{eq:zeroth-order-reflected} separates the two mechanisms clearly: the
state-dependent noise corrects the target density, while the angular current breaks
detailed balance without altering that density.

\subsection{Constraint sets and skew tensors}
The constraint sets and the state-dependent skew tensors below follow the
construction of \citet{wang2026constrained}, we define everything explicitly for the
present setting. We consider two constraint sets. The first is the centered
ball
\begin{equation}
\label{eq:ball}
K_r=\left\{x\in\mathbb{R}^{d}:\|x\|_2^{2}\le r\right\},\qquad
n(x)=\frac{x}{\|x\|_2}\ \ \text{for } x\in\partial K_r,
\end{equation}
with $r=1$. The second is a smoothed $\ell_1$ ball: for $\varepsilon>0$ and
$\Lambda>d\varepsilon$,
\begin{equation}
\label{eq:l1ball}
K=\left\{x\in\mathbb{R}^{d}: g_\varepsilon(x)\le\Lambda\right\},\qquad
g_\varepsilon(x):=\sum_{i=1}^{d}\big(x_i^{2}+\varepsilon^{2}\big)^{1/2},
\end{equation}
which is the sublevel set \eqref{eq:sublevel-domain} with
$\rho=g_\varepsilon-\Lambda$. Since $\min_{x}g_\varepsilon(x)=d\varepsilon<\Lambda$ is
attained only at $x=0\in K^\circ$, we have $\nabla g_\varepsilon(x)\neq0$ on
$\partial K=\{g_\varepsilon=\Lambda\}$, and the outward unit normal is
\begin{equation}
\label{eq:l1normal}
n(x)=\frac{\nabla g_\varepsilon(x)}{\|\nabla g_\varepsilon(x)\|_2},\qquad
\partial_i g_\varepsilon(x)=\frac{x_i}{\sqrt{x_i^{2}+\varepsilon^{2}}},\quad i=1,\dots,d .
\end{equation}
This is the case $p=1$ of the smoothed $\ell_p$ ball
$\left\{x:\sum_{i=1}^{d}(x_i^2+\varepsilon^2)^{p/2}\le\Lambda\right\}$. As $\varepsilon\to0$ the
entries of the tensor below behave like $\operatorname{sign}(x_i)|x_i|^{p-1}$,
which is $C^{1}$ only for $p\ge2$; for $p=1$ a strictly positive smoothing
parameter $\varepsilon$ is therefore needed for $J$ to be $C^{1}$.
 
On both sets we use the block cross-product construction, let $\phi\in C^\infty(\mathbb{R}^d)$
be a potential that is constant on $\partial K$. Write $d=3m$ and, for
$\ell=1,\dots,m$, let $(a_\ell,b_\ell,c_\ell):=(3\ell-2,\,3\ell-1,\,3\ell)$ and
define the block axial vector
\begin{equation}
\label{eq:axial}
k^{(\ell)}(x):=s_\ell\begin{pmatrix}\partial_{a_\ell}\phi(x)\\ \partial_{b_\ell}\phi(x)\\ \partial_{c_\ell}\phi(x)\end{pmatrix},
\qquad s_\ell\neq0 .
\end{equation}
Let $J^{(\ell)}(x)$ be the $d\times d$ matrix whose only nonzero entries form
the $3\times3$ block on the rows and columns $\{a_\ell,b_\ell,c_\ell\}$,
\begin{equation}
\label{eq:block}
\Big[J^{(\ell)}(x)\Big]_{\{a_\ell,b_\ell,c_\ell\}\times\{a_\ell,b_\ell,c_\ell\}}
=\begin{pmatrix}
0 & -k^{(\ell)}_3(x) & k^{(\ell)}_2(x)\\
k^{(\ell)}_3(x) & 0 & -k^{(\ell)}_1(x)\\
-k^{(\ell)}_2(x) & k^{(\ell)}_1(x) & 0
\end{pmatrix},
\end{equation}
i.e.\ $J^{(\ell)}(x)v=P_{I_\ell}^\top\left(k^{(\ell)}(x)\times P_{I_\ell}v\right)$ with
$I_\ell=\{a_\ell,b_\ell,c_\ell\}$, and set
\begin{equation}
\label{eq:Jsum}
J(x):=\sum_{\ell=1}^{m}J^{(\ell)}(x)
=\begin{pmatrix}
J^{(1)}(x_{1:3}) & 0_{3\times3} & \cdots & 0_{3\times3}\\
0_{3\times3} & J^{(2)}(x_{4:6}) & \cdots & 0_{3\times3}\\
\vdots & \vdots & \ddots & \vdots\\
0_{3\times3} & 0_{3\times3} & \cdots & J^{(m)}(x_{3m-2:3m})
\end{pmatrix},
\end{equation}
where $0_{3\times3}$ is the $3\times3$ zero matrix. (If $d\not\equiv0\pmod 3$, the
remaining one or two coordinates are left without a block, which does not
affect the properties below.) The two constraint sets correspond to two
choices of $\phi$.
 
\subsubsection{Ball constraint} Take $\phi(x)=\frac12\|x\|_2^{2}$, so that $\phi\equiv r/2$ on
$\partial K_r$ and $k^{(\ell)}(x)=s_\ell\,x_{I_\ell}$. With $d=9$ this is
\begin{equation}
\label{eq:Js}
J_s(x)=\begin{pmatrix}J^{(1)}(x_1,x_2,x_3)&0_{3\times3}&0_{3\times3}\\
0_{3\times3}&J^{(2)}(x_4,x_5,x_6)&0_{3\times3}\\
0_{3\times3}&0_{3\times3}&J^{(3)}(x_7,x_8,x_9)\end{pmatrix},
\end{equation}
with
\begin{equation}
\label{eq:Js-block}
J^{(1)}(x_1,x_2,x_3)=\begin{pmatrix}0&-s_1x_3&s_1x_2\\ s_1x_3&0&-s_1x_1\\ -s_1x_2&s_1x_1&0\end{pmatrix},
\end{equation}
and $J^{(2)}$, $J^{(3)}$ are obtained by replacing $(s_1;x_1,x_2,x_3)$ with
$(s_2;x_4,x_5,x_6)$ and $(s_3;x_7,x_8,x_9)$, respectively.
 
\subsubsection{Smoothed \texorpdfstring{$\ell_1$}{l1} ball constraint} Take $\phi(x)=\Lambda-g_\varepsilon(x)$, so that
$\phi\equiv0$ on $\partial K$ and, by \eqref{eq:l1normal},
\begin{equation}
\label{eq:kg}
k^{(\ell)}(x)=-s_\ell\begin{pmatrix}
x_{a_\ell}\big/\sqrt{x_{a_\ell}^{2}+\varepsilon^{2}}\\[2pt]
x_{b_\ell}\big/\sqrt{x_{b_\ell}^{2}+\varepsilon^{2}}\\[2pt]
x_{c_\ell}\big/\sqrt{x_{c_\ell}^{2}+\varepsilon^{2}}
\end{pmatrix},
\end{equation}
i.e.\ minus the scaled block gradient of $g_\varepsilon$ (the soft-sign of
$x_{I_\ell}$). With $d=9$ the resulting tensor is
\begin{equation}
\label{eq:Jg}
J_g(x)=\begin{pmatrix}J_g^{(1)}(x_1,x_2,x_3)&0_{3\times3}&0_{3\times3}\\
0_{3\times3}&J_g^{(2)}(x_4,x_5,x_6)&0_{3\times3}\\
0_{3\times3}&0_{3\times3}&J_g^{(3)}(x_7,x_8,x_9)\end{pmatrix},
\end{equation}
with
\begin{equation}
\label{eq:Jg-block}
J_g^{(1)}(x_1,x_2,x_3)=s_1\begin{pmatrix}
0 & \dfrac{x_3}{\sqrt{x_3^{2}+\varepsilon^{2}}} & -\dfrac{x_2}{\sqrt{x_2^{2}+\varepsilon^{2}}}\\[8pt]
-\dfrac{x_3}{\sqrt{x_3^{2}+\varepsilon^{2}}} & 0 & \dfrac{x_1}{\sqrt{x_1^{2}+\varepsilon^{2}}}\\[8pt]
\dfrac{x_2}{\sqrt{x_2^{2}+\varepsilon^{2}}} & -\dfrac{x_1}{\sqrt{x_1^{2}+\varepsilon^{2}}} & 0
\end{pmatrix},
\end{equation}
and again $J_g^{(2)}$, $J_g^{(3)}$ are obtained by substituting
$(s_2;x_4,x_5,x_6)$ and $(s_3;x_7,x_8,x_9)$.
 
Both tensors satisfy the requirements of our framework:
\begin{enumerate}
\item \emph{Regularity and skew-symmetry.} Each block \eqref{eq:block} is the
cross-product matrix of $k^{(\ell)}$, hence skew-symmetric, and
$k^{(\ell)}\in C^\infty$ because $\phi\in C^\infty$ (for the smoothed
$\ell_1$ ball this uses $\varepsilon>0$). Thus $J\in C^\infty(K;\mathbb{R}^{d\times d})$
and $J^\top=-J$.
\item \emph{Boundary condition $J(x)n(x)=0$.} Since $\phi$ is constant on
$\partial K$, its gradient is normal there: $\nabla\phi(x)=\|x\|_2\,n(x)$ on
$\partial K_r$ and $\nabla\phi(x)=-\|\nabla g_\varepsilon(x)\|_2\,n(x)$ on
\eqref{eq:l1ball}. Hence $k^{(\ell)}(x)\parallel P_{I_\ell}n(x)$ for every
$\ell$, the cross product $k^{(\ell)}(x)\times P_{I_\ell}n(x)$ vanishes, and
$J(x)n(x)=\sum_\ell J^{(\ell)}(x)n(x)=0$ on $\partial K$. (On the ball one has
in fact $J_s(x)x=0$ for every $x\in\mathbb{R}^d$.)
\item \emph{Divergence-free.} With $(a,b,c)=(a_\ell,b_\ell,c_\ell)$ and
$k=(k_1,k_2,k_3)^\top=s_\ell(\partial_a\phi,\partial_b\phi,\partial_c\phi)^\top$,
the nonzero entries of $J^{(\ell)}$ are $J_{ab}=-k_3$, $J_{ac}=k_2$,
$J_{bc}=-k_1$ and their negatives, so that
\[
\left(\nabla\cdot J^{(\ell)}\right)_a=\partial_bJ_{ab}+\partial_cJ_{ac}
=-\partial_bk_3+\partial_ck_2
=s_\ell\big(-\partial_b\partial_c\phi+\partial_c\partial_b\phi\big)=0,
\]
and likewise $\left(\nabla\cdot J^{(\ell)}\right)_b=\left(\nabla\cdot J^{(\ell)}\right)_c=0$, while
the remaining components vanish trivially. Summing over $\ell$ gives
$\nabla\cdot J=0$.
\end{enumerate}
The scales $s_1,s_2,s_3$ can be chosen
independently; in our experiments we take $s_1=s_2=s_3=:s$.
Finally, both domains have $C^\infty$ boundary, $a$ is Lipschitz with
$a\in[e^{-(d-1)\lambda\delta},1]$, and the drift $-a(I_d+\alpha J)\nabla U_0$ is
Lipschitz on $K$, so Assumption~\ref{ass:compact} holds and
Theorems~\ref{thm:rate-decomposition} and \ref{thm:unconstrained-ldp} and
Corollary~\ref{cor:rate-observable} apply to the continuous-time dynamics with
either tensor.

\subsection{Invariance and nonreversibility}

The infinitesimal generator of either \eqref{eq:nald} or \eqref{eq:nrald} is
\begin{equation}\label{eq:generator-decomposition}
    \calL_\alpha f=\cS f+\alpha\cA f,
\end{equation}
where
\begin{equation}\label{eq:S-A}
    \cS f=a(\Delta f-\nabla U_0\cdot\nabla f)
          =\pi^{-1}\nabla\cdot(\pi a\nabla f),
    \qquad
    \cA f=c\cdot\nabla f.
\end{equation}
In the unconstrained case all integrations by parts below are understood first for
compactly supported test functions and then extended by closure.

\begin{theorem}[Invariant measure and operator decomposition]\label{thm:invariance}
Assume that \eqref{eq:nald}, or respectively \eqref{eq:nrald}, is well posed.  In the
reflected case assume \eqref{eq:tangent-condition}.  Then $\pi$ in
\eqref{eq:target-intro} is invariant.  Moreover, for $f,g$ in a common core in
$L^2(\pi)$,
\begin{equation}\label{eq:adjoints}
    \ip{f}{\cS g}=\ip{\cS f}{g},
    \qquad
    \ip{f}{\cA g}=-\ip{\cA f}{g},
    \qquad
    \ip{f}{\calL_\alpha g}=\ip{(\cS-\alpha\cA)f}{g}.
\end{equation}
Consequently, whenever $\alpha\neq0$ and $c$ is not identically zero, the process is
nonreversible.  If its semigroup is strong Feller and topologically irreducible, $\pi$ is
its unique invariant probability measure.
\end{theorem}

\begin{proof}
The identities $\pi a=Z^{-1}e^{-U_0}$ and $\pi c=Z^{-1}J(x)\nabla\psi$,
together with Proposition~\ref{prop:state-dependent-current}, imply
\begin{equation}\label{eq:currents}
    \nabla\cdot(\pi a\nabla f)=\pi\cS f,
    \qquad
    \nabla\cdot(\pi c)=0.
\end{equation}
For smooth $f,g$, integration by parts gives
\begin{equation}\label{eq:S-form}
    \int_E f\cS g\,d\pi
    =-\int_E a\nabla f\cdot\nabla g\,d\pi
    =\int_E g\cS f\,d\pi.
\end{equation}
In the reflected case the boundary term vanishes by the Neumann condition.  Similarly,
using \eqref{eq:currents},
\begin{equation}\label{eq:A-form}
    \int_E f\cA g\,d\pi=-\int_E g\cA f\,d\pi.
\end{equation}
The boundary term in \eqref{eq:A-form} vanishes by
\eqref{eq:tangent-condition}.  Taking $f=1$ proves invariance.  The remaining claims
follow from \eqref{eq:S-form}--\eqref{eq:A-form} and standard irreducibility arguments.
\end{proof}

\begin{remark}[Why the general stream construction is useful]\label{rem:constant-anchor}
For the canonical choice \eqref{eq:canonical-psi}, the anti-symmetric drift vanishes when
$U_0$ is constant.  In contrast, the general construction \eqref{eq:a-c-def} remains
nontrivial for $U_0\equiv c_0$: one may choose any nonconstant stream potential satisfying
the boundary or decay conditions.  This distinction is essential for zeroth-order
anchored samplers.
\end{remark}

\subsection{Random time change}

The anchored construction is most transparent as a state-dependent change of clock.
Define the reference non-reversible Langevin dynamics on $E=\mathbb{R}^{d}$:
\begin{equation}\label{eq:reference-diffusion}
    dZ_s=\left[-\nabla U_0(Z_s)+\alpha e^{U_0(Z_s)}J(Z_s)\nabla\psi(Z_s)\right]ds
          +\sqrt2\,dW_s,
\end{equation}
and the non-reversible reflected
Langevin dynamics on $E=K$ with the same normal reflection term as in \eqref{eq:nrald}:
\begin{equation}\label{eq:reference-diffusion:constrained}
    dZ_s=\left[-\nabla U_0(Z_s)+\alpha e^{U_0(Z_s)}J(Z_s)\nabla\psi(Z_s)\right]ds
          +\sqrt2\,dW_s-n(Z_{s})dL_{s}.
\end{equation}
When $Z_0:=\int_Ee^{-U_0(x)}dx<\infty$, the invariant probability measure is
\begin{equation}\label{eq:pi0}
    \pi_0(dx)=Z_0^{-1}e^{-U_0(x)}\,dx.
\end{equation}
The measure $\pi_0, x\in\mathbb{E}$ is invariant
for \eqref{eq:reference-diffusion} when $E=\mathbb{R}^{d}$
and for \eqref{eq:reference-diffusion:constrained} when $E=K$.
We have the following random-time change representation.

\begin{proposition}[Random-time change representation]\label{prop:time-change}
Let
\begin{equation}\label{eq:clock}
    T_s=\int_0^s\frac{dr}{a(Z_r)},
    \qquad
    \tau_t=\inf\{s\geq0:T_s>t\}.
\end{equation}
Assume $T_s\to\infty$ almost surely.  Then $X_t=Z_{\tau_t}$ has infinitesimal generator
$\calL_\alpha$ in \eqref{eq:generator-decomposition}, and hence solves
\eqref{eq:nald} or \eqref{eq:nrald} up to the standard identification of the time-changed
Brownian motion and boundary regulator.  Its invariant measure satisfies
\begin{equation}\label{eq:time-change-invariant}
    \pi(dx)\ \propto\left(a(x)\right)^{-1}e^{-U_0(x)}\,dx=e^{-U(x)}\,dx.
\end{equation}
\end{proposition}

\begin{proof}
The infinitesimal generator of \eqref{eq:reference-diffusion} is
\begin{equation*}
    \widetilde{\calL}_\alpha f
    =\Delta f-\nabla U_0\cdot\nabla f
      +\alpha e^{U_0}J(x)\nabla\psi\cdot\nabla f.
\end{equation*}
The inverse additive-functional time change in \eqref{eq:clock} multiplies this infinitesimal generator
by $a$.  Since $ae^{U_0}=e^U$, we obtain \eqref{eq:S-A}.  The invariant measure of a
time change with infinitesimal generator $a\widetilde{\calL}_\alpha$ is obtained by multiplying the
invariant speed measure by $a^{-1}$.  This gives \eqref{eq:time-change-invariant}
whether or not $Z_0$ is finite.
\end{proof}

\begin{remark}
The condition $T_s\to\infty$ as $s\to\infty$, which ensures that the inverse
clock $\tau_t$ is finite for every $t$, is automatic in the common
positive-recurrent case.
If $Z_0<\infty$ and the reference diffusion is ergodic, then the ergodic theorem gives
\begin{equation*}
    \frac{T_s}{s}\longrightarrow\int_Ea^{-1}d\pi_0=\frac{Z}{Z_0}>0
    \qquad\text{as }s\to\infty,\quad\text{almost surely}.
\end{equation*}
In particular, this applies on the compact reflected domain and to the Student family in
Corollary~\ref{cor:student-ldp}.
\end{remark}

\section{Non-Asymptotic Convergence Analysis}\label{sec:non:asymptotic}

Next, let us study the convergence speed
of the continuous-time process NALD $X_{t}$ in \eqref{eq:nald}
and NRALD $X_{t}$ in \eqref{eq:nrald}
to their invariant distributions.
Before we proceed, let us introduce two notions
of non-asymptotic convergence bound here.

First, we let $\lambda_{\alpha}$ be the spectral gap
of $\mathcal{L}$ in $L^{2}(\pi)$.
That is, 
\begin{equation}
\int_{E}\left(\mathbb{E}^{x}\left[f(X_{t})\right]-\pi(f)\right)^{2}\pi(x)dx
=\Vert P_{t}f-\pi(f)\Vert^{2}
\leq
C_{\alpha}\Vert f-\pi(f)\Vert^{2}e^{-2\lambda_{\alpha}t},
\end{equation}
for some constant $C_{\alpha}$.
In particular, $\lambda_{0}$ denotes
the spectral gap when $\alpha=0$, i.e. 
the case of the anchored Langevin dynamics.
We assume the existence of the spectral gap of $\mathcal{L}$ when $\alpha=0$, i.e. it is positive.

\begin{assumption}\label{assump: spectral gap}
Assume that $\lambda_{0}>0$.
\end{assumption}

It is worth noting that the existence of a spectral gap can be derived from the
logarithmic Sobolev inequality, as discussed in \cite{Bakry2013}.
Furthermore, the logarithmic Sobolev inequality itself can be established under
the weaker condition of the existence of a Lyapunov function for the diffusion
generator \citep{bakry2008simple, cattiaux2010note}. Explicit lower bounds for
the logarithmic Sobolev inequality in convex bounded domain have been explored
in works such as \cite{wang1997estimation}.

Next, we introduce the rate of convergence of $p(t,x,y)$,
the transition probability density of $X_{t}$, to the Gibbs distribution $\pi$ in the variational norm, which is equivalent to the total variation (TV) distance
up to a $1/2$ factor. Let us define:
\begin{equation}
\rho_{\alpha}:=\sup\left\{\rho:\int_{E}|p(t,x,y)-\pi(y)|dy\leq g(x)e^{-\rho t}\right\},
\end{equation}
for some $g(x)$ that may depend on $J$. In particular, $\rho_{0}$ denotes
the case when $\alpha=0$, i.e. the case of the anchored Langevin dynamics.

Before we proceed, let us 
introduce the Dirichlet form:
\begin{equation}\label{defn:Dirichlet}
\mathcal{E}(f,f):=-\int_{E}f\mathcal{L}f d\pi=\int_{E}a\Vert\nabla f\Vert^2 d\pi.
\end{equation}

We will obtain a non-asymptotic bound
for the convergence of the distribution of $X_{t}$
to the invariant distribution $\pi$ in \textit{total variation} (TV) distance.

\subsection{Unconstrained case}

Next, we consider the unconstrained case $E=\mathbb{R}^{d}$. We have the following finite-time non-asymptotic convergence guarantees in the TV distance.

\begin{theorem}\label{thm:TV:unconstrained}
Let $E=\mathbb{R}^{d}$ and suppose Assumptions~\ref{ass:unconstrained-local}
and \ref{assump: spectral gap} hold.
There exists some locally bounded function $g(x)$ such that
for any $X_{0}=x\in\mathbb{R}^{d}$,
\begin{equation}
\mathrm{TV}(\mathrm{Law}(X_{t}),\pi)
\leq g(x)e^{-\rho_{\alpha}t},
\end{equation}
where $\rho_{\alpha}\geq\rho_{0}>0$. 
\end{theorem}

\begin{proof}
We adapt the proof of \cite[Theorems~4--5]{HHS05} to the locally uniformly
elliptic infinitesimal generator $\calL_\alpha=a\Delta+b_\alpha\cdot\nabla$.
Write
\[
    q_t(x,y):=\frac{p(t,x,y)}{\pi(y)}.
\]
For $t\geq1$, the Chapman--Kolmogorov equation and the adjoint semigroup give
\begin{align}
    \int_{\R^d}\abs{q_t(x,y)-1}\pi(y)\,dy
    &=
    \int_{\R^d}
      \abs{P_{t-1}^*\left(q_1(x,\cdot)-1\right)(y)}
      \pi(y)\,dy\notag\\
    &\leq
    \left\|P_{t-1}^*\left(q_1(x,\cdot)-1\right)\right\|_{L^2(\pi)}
      \notag\\
    &\leq
    C_\alpha e^{-\lambda_\alpha(t-1)}
    \left\|q_1(x,\cdot)-1\right\|_{L^2(\pi)}.
    \label{eq:unconstrained-TV-L2}
\end{align}
Here the last step uses the fact that $P_t$ and $P_t^*$ have the same
$L^2(\pi)$ operator norm on the mean-zero subspace.

It remains to justify that the last norm is finite and locally bounded in $x$.
Under Assumption~\ref{ass:unconstrained-local}, $a$ is bounded above and away
from zero on every compact set and $b_\alpha$ is locally bounded.  The local
parabolic Harnack inequality used in \cite[Theorem~4]{HHS05} therefore applies
to $\calL_\alpha$.  It yields a locally bounded function $H_\alpha$ such that
\begin{equation}\label{eq:local-kernel-bound}
    \sup_{y\in\mathbb{R}^{d}} q_1(x,y)\leq H_\alpha(x).
\end{equation}
Since $\int_{\mathbb{R}^{d}} q_1(x,y)\pi(y)\,dy=1$, we obtain
\[
    \int_{\mathbb{R}^{d}} \left(q_1(x,y)\right)^2\pi(y)\,dy
    \leq H_\alpha(x),
\]
so the prefactor in \eqref{eq:unconstrained-TV-L2} is locally bounded.
This proves $\rho_\alpha\geq\lambda_\alpha$.  The Dirichlet-form argument of
Lemma~\ref{lem:lambda}, which does not use compactness of the state space,
gives $\lambda_\alpha\geq\lambda_0$.

For the reversible process, the function furnished by
\eqref{eq:unconstrained-TV-L2} is locally bounded and hence belongs to
$L^1_{\rm loc}(\pi)$.  The compactly supported test-function argument of
\cite[Theorem~5]{HHS05} then gives the reverse comparison
$\rho_0\leq\lambda_0$.  Together with the preceding inequality for
$\alpha=0$, this yields $\rho_0=\lambda_0$, and hence
\[
    \rho_\alpha\geq\lambda_\alpha\geq\lambda_0=\rho_0>0.
\]
This completes the proof.
\end{proof}


\subsection{Constrained case}

Next, we consider the constrained case $E=K$. We have the following finite-time non-asymptotic convergence guarantees in the TV distance.

\begin{theorem}\label{thm:TV}
Let $E=K$. There exists some constant $\mathcal{K}>0$ such that
for any $X_{0}=x\in K$,
\begin{equation}
\mathrm{TV}(\mathrm{Law}(X_{t}),\pi)
\leq\mathcal{K}e^{-\rho_{\alpha}t},
\end{equation}
where $\rho_{\alpha}\geq\rho_{0}>0$. 
\end{theorem}

Theorem~\ref{thm:TV} is inspired by the result for the convergence of non-reversible Langevin dynamics
for the unconstrained case in \cite{HHS05} 
and the convergence of non-reversible skew-reflected Langevin dynamics for the constrained domain 
in \cite{DFTWZ2025}. Since our $K$ is bounded, we are able to obtain
a convergence bound in Theorem~\ref{thm:TV} uniformly in the initial starting point $X_{0}=x\in K$.
To prove Theorem~\ref{thm:TV}, we will establish a sequence
of technical lemmas.

\begin{lemma}\label{lem:lambda}
$\lambda_{\alpha}\geq\lambda_{0}$.    
\end{lemma}

\begin{proof}
We follow the similar argument
as in the proof of Theorem~1 in \cite{HHS05} for the unconstrained case.
If $f\in\mathcal{D}(\mathcal{L})$, then $f\in\mathcal{D}(\mathcal{E})$ and
\begin{equation}\label{key:ineq}
\mathcal{E}(f,f)=-\int_{E}(\mathcal{L}f)fd\pi,    
\end{equation}
(see page 124 in \cite{stannat1999nonsymmetric} and \eqref{defn:Dirichlet}), 
where $\mathcal{E}(\cdot,\cdot)$ is the Dirichlet form in \eqref{defn:Dirichlet}.
For any $f$ with $\Vert f\Vert=1$ and $\pi(f)=0$, let $h(t):=\Vert P_{t}f\Vert^{2}$. 
Then $h(0)=1$ and by \eqref{key:ineq}, 
we have
\begin{equation}
h'(t)=2\int_{E}(\mathcal{L}P_{t}f)(P_{t}f)d\pi
\leq -2\mathcal{E}(P_{t}f,P_{t}f)
\leq -2\lambda_{0}h(t).
\end{equation}
This implies that the operator norm $\Vert P_{t}\Vert$ in the space $\{f:f\in L^{2}(\pi),\pi(f)=0\}$ is less than or equal to $e^{-\lambda_{0}t}$. Hence, $\lambda_{\alpha}\geq\lambda_{0}$.
This completes the proof.
\end{proof}

Note that Lemma~\ref{lem:lambda} applies
to both constrained and unconstrained case, and in the proof of Lemma~\ref{lem:lambda}, one can take either $E=K$ or $E=\mathbb{R}^{d}$.
Next, we show the following lemma.

\begin{lemma}\label{lem:TV}
There exists some constant $\mathcal{K}>0$ such that
for any $X_{0}=x\in K$,
\begin{equation}
\mathrm{TV}(\mathrm{Law}(X_{t}),\pi)
\leq \mathcal{K}e^{-\rho_{\alpha}t},
\end{equation}
where $\rho_{\alpha}\geq\lambda_{\alpha}$.
\end{lemma}

\begin{proof}
The proof is inspired by that of Theorem~4 in \cite{HHS05} 
for the unconstrained case. 
Let $p_{t}(x,y):=p(t,x,y)/\pi(y)$ for any $x,y\in K$, 
where $p(t,x,y)$ denotes the probability density function
for $X_{t}$ given that $X_{0}=x\in K$. 
By the definition of the TV distance, we have
\begin{equation}
\mathrm{TV}(\mathrm{Law}(X_{t}),\pi)
=\int_{K}|p(t,x,y)-\pi(y)|dy
=\int_{K}|p_{t}(x,y)-1|\pi(y)dy.
\end{equation}
By letting $p^{\ast}_{t}(\cdot,\cdot)$ denote the adjoint process, 
we can compute that for any $t\geq 1$,
\begin{align*}
\int_{K}|p_{t}(x,y)-1|\pi(y)dy
&=\int_{K}\left|\int_{K}\left(p_{1}(x,z)p_{t-1}(z,y)-1\right)\pi(z)dz\right|\pi(y)dy
\\
&=\int_{K}\left|\int_{K}(p_{1}(x,z)p^{\ast}_{t-1}(y,z)-1)\pi(z)dz\right|\pi(y)dy
\\
&=\int_{K}\left|\int_{K}p^{\ast}_{t-1}(y,z)p_{1}(x,z)\pi(z)dz-\int_{K}p_{1}(x,z)\pi(z)dz\right|\pi(y)dy
\\
&=\int_{K}|P^{\ast}_{t-1}(p_{1}(x,\cdot))(y)-\pi(p_{1}(x,\cdot))|\pi(y)dy
\\
&\leq\left(\int_{K}|P^{\ast}_{t-1}(p_{1}(x,\cdot))(y)-\pi(p_{1}(x,\cdot))|^{2}\pi(y)dy\right)^{1/2}
\\
&\leq C\Vert p_{1}(x,\cdot)-1\Vert e^{-\lambda_{\alpha}t},
\end{align*}
where we used Cauchy-Schwarz inequality and $C$ is some constant.  By applying the regularity estimates for oblique parabolic equations from \cite{lieberman1990holder}, we conclude that \(p_1(x, \cdot)\) is H\"{o}lder continuous on $K$. Consequently, \(p_1(x, \cdot) \in L^2(\pi)\), and thus, the final inequality holds.
This also establishes that $\rho_{\alpha}\geq\lambda_{\alpha}$. 
\end{proof}

Next, we show that $\lambda_{0}=\rho_{0}$.

\begin{lemma}\label{lem:lambda:0:rho:0}
$\lambda_{0}=\rho_{0}$.   
\end{lemma}

\begin{proof}
The proof is inspired by that of Theorem~5 in \cite{HHS05} 
for the unconstrained case.
When $J=0$, $X_{t}$ is reversible and $P_{t}$ is self-adjoint in $L^{2}(\pi)$.
Given $f$ with $\pi(f)=0$ and $f\in C^{\infty}$, we have
\begin{align*}
\Vert P_{t}f\Vert^{2}&=\pi\left(fP_{2t}f\right)    
\\
&=\int_{K}f(x)\left(\int_{K}(p_{2t}(x,y)-1)f(y)\pi(y)dy\right)\pi(x)dx
\\
&\leq\sup_{x\in K}\Vert f(x)\Vert^{2}\int_{K}\int_{K}|p_{2t}(x,y)-1|\pi(y)dy\pi(x)dx
\\
&\leq\sup_{x\in K}\Vert f(x)\Vert^{2}\int_{K}\mathcal{K}e^{-2\rho t}\pi(x)dx
\\
&=\sup_{x\in K}\Vert f(x)\Vert^{2}\mathcal{K}e^{-2\rho t},
\end{align*}
where we applied Lemma~\ref{lem:TV}.

By Lemma~2.2. in \cite{Rockner2001}, for $s\leq t$ and $\pi(f^{2})=1$, 
\begin{equation}
\Vert P_{s}f\Vert^{2}\leq\left(\Vert P_{t}f\Vert^{2}\right)^{s/t}\leq   
\left(\sup_{x\in K}\Vert f(x)\Vert^{2}\mathcal{K}\right)^{s/t}e^{-2\rho_{0}s},
\end{equation}
which implies 
\begin{equation}\label{take:derivative:eqn}
h(s)\leq\left(\sup_{x\in K}\Vert f(x)\Vert^{2}\mathcal{K}\right)^{s/t}e^{-2\rho_{0}s},
\end{equation}
where $h(t):=\Vert P_{t}f\Vert^{2}$.
Since $h(0)=1$, the equality holds at $s=0$ in \eqref{take:derivative:eqn}. By taking derivative in \eqref{take:derivative:eqn} with respect to $s$
and letting $s=0$, we get
\begin{equation}
-2\mathcal{E}(f,f)\leq\frac{1}{t}\log\left(\sup_{x\in K}\Vert f(x)\Vert^{2}\mathcal{K}\right)-2\rho_{0},
\end{equation}
where $\mathcal{E}(f,f)$ is the Dirichlet form \eqref{defn:Dirichlet}, and we used the fact
that $h'(0)=-2\mathcal{E}(f,f)$.
By letting $t\rightarrow\infty$, we have
\begin{equation}
\mathcal{E}(f,f)\geq\rho_{0}.
\end{equation}
This completes the proof.
\end{proof}

Now, we are finally ready to prove Theorem~\ref{thm:TV}.

\begin{proof}[Proof of Theorem~\ref{thm:TV}]
It follows from Lemma~\ref{lem:TV}
that there exists some constant $\mathcal{K}>0$ such that
for any $X_{0}=x\in K$,
\begin{equation}
\mathrm{TV}(\mathrm{Law}(X_{t}),\pi)
\leq \mathcal{K}e^{-\rho_{\alpha}t},
\end{equation}
where $\rho_{\alpha}\geq\lambda_{\alpha}$.
We recall from Lemma~\ref{lem:lambda}
that $\lambda_{\alpha}\geq\lambda_{0}$
and from Lemma~\ref{lem:lambda:0:rho:0}
that $\lambda_{0}=\rho_{0}$. 
Hence, we conclude that $\rho_{\alpha}\geq\lambda_{\alpha}\geq\lambda_{0}=\rho_{0}>0$, where we also used $\lambda_{0}>0$ from Assumption~\ref{assump: spectral gap}.
The proof is complete.
\end{proof}

\section{Large Deviations of the Empirical Measure}\label{sec:ldp}

In this section, we study the acceleration of non-reversible anchored Langevin dynamics (NALD)
and non-reversible reflected anchored Langevin dynamics (NRALD) through the lens of large deviations analysis.
For a trajectory of either anchored diffusion, define its empirical measure
\begin{equation}\label{eq:empirical-measure}
    L_t=\frac1t\int_0^t\delta_{X_s}\,ds.
\end{equation}
This section compares the large deviations of $\mathbb{P}(L_t\in\cdot)$ for $\alpha\neq0$ with those of the
reversible anchored process $\alpha=0$.  
Informally, $\mathbb{P}(L_{t}\in\cdot)$ satisfies a large deviation principle
with rate function $I(\cdot)$ if for any probability measure $\nu$ supported on $E$,
$\mathbb{P}(L_{t}\approx\nu)=e^{-tI(\nu)+o(t)}$,
as $t\rightarrow\infty$, where $I(\nu)\geq 0$ and $I(\nu)=0$ if and only if $\nu=\pi$.
In other words, when $\nu\neq\pi$, $\mathbb{P}(L_{t}\approx\nu)$ is exponentially
small as time $t\rightarrow\infty$, and the rate function $I(\nu)$ characterizes
how small this probability is. This suggests the larger the value of the rate function,
the smaller the probability the empirical measure deviates away from the invariant distribution,
and hence one expects faster convergence.
Formally, a sequence of Borel probability measures $\mathbb{P}_{t}$ on a polish space $\mathcal{X}$ is said to satisfy a \textit{large deviation principle} (LDP) in with speed $t$ and rate function $I: \mathcal{X} \to [0, \infty]$ 
if $I$ is non-negative, lower semicontinuous and for any measurable set $A$, we have
\begin{equation}
- \inf_{\nu\in A^\circ} I(\nu) \leq \liminf_{t\to \infty} \frac{1}{t}\log \mathbb{P}_{t}(A) \leq
\limsup_{t\to \infty} \frac{1}{t} \log \mathbb{P}_{t}(A) \leq - \inf_{\nu\in \bar A} I(\nu) \,,
\end{equation}
where $A^\circ$ denotes the interior of $A$ and $\bar A$ its closure.
We will specify the underlying space $\mathcal{X}$ and its topology
for the constrained case $E=K$ and unconstrained case $E=\mathbb{R}^{d}$ carefully in the following analysis.
We refer to \cite{Varadhan1,dembo1998large} for a general theory and applications of large deviations. 

We begin with the compact reflected setting,
where the existence of the level-two large deviation principle follows from standard
Donsker--Varadhan theory \citep{donsker1975asymptotic,dembo1998large}.  We then state an
unconstrained version under the weighted-topology assumptions of
\cite{ferre2020large,Yao2025LDP}.

\subsection{Weighted Sobolev seminorms}

Let $\nu\in\cP(E)$.  For any smooth function $h$, define
\begin{equation}\label{eq:H1-norm}
    \norm{h}_{H^1_a(\nu)}^2
    :=\int_E a\abs{\nabla h}^2\,d\nu.
\end{equation}
For a distribution $q$ which annihilates constants, define the dual seminorm
\begin{equation}\label{eq:Hminus1-norm}
    \norm{q}_{H^{-1}_a(\nu)}^2
    :=\sup_{h}\left\{2\int_E hq\,d\nu
      -\int_E a\abs{\nabla h}^2\,d\nu\right\},
\end{equation}
where the supremum is taken over a core of smooth functions, followed by closure.  When
$q=\cA v$, the pairing in \eqref{eq:Hminus1-norm} is interpreted weakly.

\subsection{The compact reflected case}

We use the following regularity assumptions for a direct spectral formulation.

\begin{assumption}[Compact reflected setting]\label{ass:compact}
The domain $K$ is bounded and connected, $\partial K\in C^{2,\delta}$ for some
$\delta\in(0,1)$, and the coefficients of $\calL_\alpha$ are $C^\delta$ up to the
boundary.  The function $a$ is bounded above and below by positive constants,
$c\cdot n=0$ on $\partial K$, and the reflected diffusion is well posed.
\end{assumption}

Under Assumption~\ref{ass:compact}, 
the empirical measure satisfies
a large deviation principle, 
and the rate function can be decomposed 
as the sum of the rate function
of the reversible part and a non-negative term, which implies acceleration. 

\begin{theorem}[Constrained LDP and Rate-function decomposition]\label{thm:rate-decomposition}
Suppose Assumption~\ref{ass:compact} holds.  
Then, $(L_{t})_{t\geq 0}$ satisfies an
LDP on $\cP(\overline K)$ equipped with the weak topology with speed $t$ and rate function
\begin{equation}\label{eq:DV}
    \cI_\alpha(\nu)
    :=-\inf_{u\in\calD_{N,+}}
      \int_K\frac{\calL_\alpha u}{u}\,d\nu.
\end{equation}
Here $\calD_{N,+}$ is the positive Neumann core defined in
\eqref{eq:neumann-domain}--\eqref{eq:positive-neumann-domain}.
Let $\nu=e^v\pi$ be a probability measure
such that $v\in H^1_a(\nu)$ and $\cA v\in H^{-1}_a(\nu)$.  Then,
\begin{equation}\label{eq:rate-decomposition}
    \cI_\alpha(\nu)
    =\frac14\int_K a\abs{\nabla v}^2\,d\nu
    +\frac{\alpha^2}{4}\norm{\cA v}_{H^{-1}_a(\nu)}^2.
\end{equation}
In particular,
\begin{equation}\label{eq:rate-order}
    \cI_\alpha(\nu)\geq\cI_0(\nu):=\frac14\int_K a\abs{\nabla v}^2\,d\nu,
\end{equation}
with strict inequality whenever $\cA v$ is nonzero in $H^{-1}_a(\nu)$.
\end{theorem}

\begin{proof}
We adapt the argument of
\cite[Theorem~7 and Proposition~9]{wang2026constrained}
for non-reversible reflected Langevin dynamics on a compact domain.  For each
$V\in C(\overline K)$, the Feynman--Kac semigroup generated by
$\calL_\alpha+V$ is strongly positive and compact.  Its principal eigenvalue
$\lambda_\alpha(V)$ is finite, convex, and G\^ateaux differentiable.  The
G\"artner--Ellis argument in \cite{wang2026constrained} therefore gives the level-two LDP;
exponential tightness is automatic because $\cP(\overline K)$ is compact.

The proof in \cite{wang2026constrained} is written for unit diffusivity and a
skew perturbation of the drift.  Here uniform ellipticity and normal Neumann
reflection give the same compactness and positivity properties, while the
more general current $c$ has exactly the two properties used in the
integration by parts: $\nabla\cdot(\pi c)=0$ and $c\cdot n=0$.  Thus the
proof of Proposition~9 in \cite{wang2026constrained} carries over with the carr\'e du champ
$\Gamma(f)=a\abs{\nabla f}^2$, yielding
\begin{equation}
    \cI_\alpha(\nu)
    =-\inf_{u\in\calD_{N,+}}
      \int_K\frac{\calL_\alpha u}{u}\,d\nu.
\end{equation}

Write $u=e^\phi$.  Since $\cA$ is first order and the carr\'e du champ of $\cS$ is
$a\abs{\nabla\phi}^2$,
\begin{equation}\label{eq:exponential-generator}
    \frac{\calL_\alpha e^\phi}{e^\phi}
    =\calL_\alpha\phi+a\abs{\nabla\phi}^2.
\end{equation}
Using $d\nu=e^v d\pi$, symmetry of $\cS$, and anti-symmetry of $\cA$, we obtain
\begin{align}
    -\int_K\cS\phi\,d\nu
      &=\int_K a\nabla\phi\cdot\nabla v\,d\nu,\label{eq:S-ibp-v}
\end{align}
and
\begin{align}
    -\int_K\cA\phi\,d\nu
      &=\int_K\phi\,\cA v\,d\nu.\label{eq:A-ibp-v}
\end{align}
Moreover,
\begin{align}
\int_K v\cA v\,d\nu=0.  
\end{align}
Indeed, choose the antiderivative
$F(r)=(r-1)e^r$, for which $F'(r)=re^r$.  Since $d\nu=e^v d\pi$
and $\cA$ is a first-order derivation,
\[
    \int_Kv\cA v\,d\nu
    =\int_Kve^v\cA v\,d\pi
    =\int_K\cA(F(v))\,d\pi=0,
\]
where the last equality follows from the invariance of $\pi$ under $\cA$.

Substituting \eqref{eq:exponential-generator}--\eqref{eq:A-ibp-v} into
\eqref{eq:DV}, and setting $\phi=v/2+h$, gives
\begin{align}
    \cI_\alpha(\nu)
    &=\frac14\int_K a\abs{\nabla v}^2d\nu
      +\sup_h\left\{\alpha\int_Kh\cA v\,d\nu
      -\int_Ka\abs{\nabla h}^2d\nu\right\}\notag\\
    &=\frac14\int_K a\abs{\nabla v}^2d\nu
      +\frac{\alpha^2}{4}\norm{\cA v}_{H^{-1}_a(\nu)}^2,
\end{align}
where the second equality follows by rescaling $h$ in
\eqref{eq:Hminus1-norm}.  This proves the result.
\end{proof}

\begin{remark}[Lower-semicontinuous extension]
The Donsker--Varadhan functional extends beyond the regular class considered
in Theorem~\ref{thm:rate-decomposition} through lower semicontinuous
relaxation on the extended generator domain; see
\cite[Theorem~3.3]{ferre2020large}.  The explicit decomposition
\eqref{eq:rate-decomposition}, however, is asserted here only for measures
$\nu=e^v\pi$ satisfying the regularity assumptions of the theorem.
\end{remark}

For an observable $g$, define the contraction of the rate function by
\begin{equation}\label{eq:observable-rate}
    \cI_{\alpha,g}(z)
    :=\inf\left\{\cI_\alpha(\nu):\nu\in\cP(K),\ \int_{K}g\,d\nu=z\right\}.
\end{equation}

The pointwise ordering of the level-two rate functions is preserved under this
contraction.  Hence the same circulation that accelerates occupation-measure
concentration also improves the large-deviation rate for every continuous
ergodic observable.

\begin{corollary}[Occupation-measure acceleration]\label{cor:rate-observable}
Under the assumptions of Theorem~\ref{thm:rate-decomposition}, let $g$ be continuous on
$\overline K$.  Then
\begin{equation}\label{eq:observable-order}
    \cI_{\alpha,g}(z)\geq\cI_{0,g}(z),
\end{equation}
for every $z$ for which the constraint set is nonempty.  The inequality is strict if every
minimizer $\nu=e^v\pi$ of the reversible contraction at $z$ satisfies $\cA v\neq0$ in
$H^{-1}_a(\nu)$.
\end{corollary}

\begin{proof}
Set
\[
    C_z:=\left\{\nu\in\cP(K):\int_Kg\,d\nu=z\right\},
    \qquad
    m_\alpha:=\inf_{\nu\in C_z}\cI_\alpha(\nu).
\]
Because $g$ is continuous, $C_z$ is closed.  Equation
\eqref{eq:rate-order} gives
\[
    m_\alpha
    =\inf_{\nu\in C_z}\cI_\alpha(\nu)
    \geq\inf_{\nu\in C_z}\cI_0(\nu)
    =m_0.
\]
For strictness, suppose instead that $m_\alpha=m_0$.  Goodness of
$\cI_\alpha$ and closedness of $C_z$ give a minimizer
$\nu_\alpha\in C_z$. 
Hence,
\[
    m_0
    \leq\cI_0(\nu_\alpha)
    \leq\cI_\alpha(\nu_\alpha)
    =m_\alpha
    =m_0.
\]
Thus, $\nu_\alpha$ is also a minimizer of the reversible contraction, and
\eqref{eq:rate-decomposition} forces
$\cA v_\alpha=0$ in $H^{-1}_a(\nu_\alpha)$ when
$d\nu_\alpha=e^{v_\alpha}d\pi$.  This contradicts the stated strictness
condition and proves the claim.
\end{proof}

\begin{remark}
The comparison \eqref{eq:rate-order} concerns concentration of long-time occupation
measures and ergodic averages.  It does not, by itself, imply a pointwise ordering of
finite-time total variation or Wasserstein distances.
\end{remark}

\subsection{The unconstrained case}

On $\R^d$, exponential tightness is no longer automatic.  We formulate assumptions in a
form adapted from the weighted-topology framework of \cite{ferre2020large}.

\begin{assumption}[Unconstrained setting]\label{ass:unconstrained-ldp}
We assume the following holds.
\begin{enumerate}[label=\textup{(\roman*)},leftmargin=2.5em]
    \item The drift $b_\alpha$ and the diffusion coefficient
    $\sqrt{2a}I_{d}$ are smooth and, together with all their derivatives, have
    at most polynomial growth.  The H\"ormander bracket condition holds:
    the Lie algebra generated by the diffusion vector fields
    $\sqrt{2a}\,\partial_1,\ldots,\sqrt{2a}\,\partial_d$ and their iterated
    brackets with the drift spans $\R^d$ at every point
    \citep{hormander1967hypoelliptic}.

    \item For every $x,y\in\R^d$ and $T>0$, there is a control
    $u\in L^2([0,T];\R^d)$ for which the controlled equation
    \[
        \dot z_t=b_\alpha(z_t)+\sqrt{2a(z_t)}\,u_t,
        \qquad z_0=x,
    \]
    has a solution satisfying $z_T=y$.

    \item There exist
    $\cW,\overline{\cW}\in C^2(\R^d;[1,\infty))$ such that $\cW$ has
    compact level sets and
    \begin{equation}\label{eq:witten-lyapunov}
        \Psi_\alpha:=-\frac{\calL_\alpha\cW}{\cW}
    \end{equation}
    has compact level sets.  Moreover, for constants $C_1>0$ and
    $C_2\in\R$,
    \begin{equation}\label{eq:witten-auxiliary}
        \overline{\cW}^{\,2}\leq C_1\cW,
        \qquad
        \Psi_\alpha\asymp
        -\frac{\calL_\alpha\overline{\cW}}{\overline{\cW}},
        \qquad
        -2\frac{\calL_\alpha\overline{\cW}}{\overline{\cW}}
        \leq\Psi_\alpha+C_2.
    \end{equation}
    Here $f\asymp g$ means equivalence up to positive multiplicative and
    additive constants.

    \item The function
    $\kappa\in C^\infty(\R^d;[1,\infty))$ has at most polynomial growth
    together with its derivatives and satisfies
    \[
        \frac{\kappa(x)}{\Psi_\alpha(x)}\longrightarrow0
        \qquad\text{as }\abs{x}\to\infty.
    \]
    It is either bounded, or has compact level sets and satisfies
    \begin{equation}\label{eq:kappa-multiplicative}
        \calL_\alpha(\kappa\cW)\leq C_3\kappa\cW,
    \end{equation}
    for some $C_3\in\R$.
\end{enumerate}
\end{assumption}

The preceding hypotheses are tailored to the weighted-topology theorem but
can be cumbersome to check directly.  The next proposition gives a concrete
set of standard ellipticity and dissipativity conditions that verifies all
four items for the canonical current.

\begin{proposition}[A sufficient condition for the unconstrained LDP assumptions]
\label{prop:unconstrained-ldp-sufficient}
Consider the canonical perturbation $\psi=e^{-U_0}$.  Suppose, after an
irrelevant additive shift, that $U_0\geq0$ and that:
\begin{enumerate}[label=\textup{(\roman*)},leftmargin=2.5em]
    \item $U_0\in C^\infty(\R^d)$, all its derivatives have at most
    polynomial growth, $\nabla U_0$ is globally Lipschitz, and there are
    $m>0$ and $b\geq0$ such that
    \begin{equation}\label{eq:U0-strong-dissipativity}
        x\cdot\nabla U_0(x)\geq m\abs{x}^2-b;
    \end{equation}
    \item $a\in C^\infty(\R^d)$, its derivatives have at most polynomial
    growth, and
    \begin{equation}\label{eq:a-uniform-bounds}
        0<a_-\leq a(x)\leq a_+<\infty;
    \end{equation}
    \item $J$ is a smooth pointwise skew-symmetric divergence-free matrix
    field whose derivatives have at most polynomial growth.
\end{enumerate}
Then Assumption~\ref{ass:unconstrained-ldp} holds with the bounded weight
$\kappa=1$.
\end{proposition}

\begin{proof}
For the canonical perturbation,
\[
    \calL_\alpha
    =a\left(\Delta-(I_{d}+\alpha J)\nabla U_0\cdot\nabla\right).
\]
The coefficient assumptions give the required polynomial growth.  By
\eqref{eq:a-uniform-bounds}, the diffusion coefficient has full rank
everywhere: the vector fields
$\sqrt{2a}\,\partial_1,\ldots,\sqrt{2a}\,\partial_d$ already span
$\R^d$ without taking brackets.  Thus the H\"ormander condition in
Assumption~\ref{ass:unconstrained-ldp}(i) holds; see
\cite{hormander1967hypoelliptic}.

To verify item (ii), fix $x,y\in\R^d$ and $T>0$, and choose a $C^1$ path
$\gamma:[0,T]\to\R^d$ with $\gamma_0=x$ and $\gamma_T=y$.  Define
\begin{align}\label{u:control}
    u_t
    =\frac{\dot\gamma_t-b_\alpha(\gamma_t)}
           {\sqrt{2a(\gamma_t)}}.
\end{align}
The image of $\gamma$ is compact and $a$ is bounded away from zero there,
so $u\in L^2([0,T];\R^d)$.  
We can re-write \eqref{u:control} as
\[
b_\alpha(\gamma_t)+\sqrt{2a(\gamma_t)}\,u_t
    =\dot\gamma_t.
\]
Hence $z_t=\gamma_t$ solves the controlled equation, starts from $x$, and
reaches $y$ at time $T$, which is precisely the controllability required in
Assumption~\ref{ass:unconstrained-ldp}(ii).

Condition \eqref{eq:U0-strong-dissipativity} implies both that $U_0$ has
compact level sets and that
\[
    \abs{\nabla U_0(x)}
    \geq m\abs{x}-\frac{b}{\abs{x}}
    \longrightarrow\infty
    \qquad\text{as }\abs{x}\to\infty.
\]
Choose $\eta\in(0,1)$ and then $\varepsilon>0$ sufficiently small that
\[
    2\varepsilon<\eta,
    \qquad
    2\varepsilon(1-\varepsilon)<\eta(1-\eta),
\]
and set
\[
    \cW=e^{\eta U_0},
    \qquad
    \overline{\cW}=e^{\varepsilon U_0}.
\]
Skew-symmetry gives $\cA U_0=0$, and hence
\begin{equation}\label{eq:exponential-U0-lyapunov}
    -\frac{\calL_\alpha e^{rU_0}}{e^{rU_0}}
    =ar\left\{(1-r)\abs{\nabla U_0}^2-\Delta U_0\right\},
    \qquad r\in(0,1).
\end{equation}
Since $\nabla U_0$ is globally Lipschitz, $\Delta U_0$ is bounded.
Equations \eqref{eq:a-uniform-bounds} and
\eqref{eq:exponential-U0-lyapunov} therefore show that
$\Psi_\alpha$ has compact level sets and that
\[
    \Psi_\alpha
    \asymp
    -\frac{\calL_\alpha\overline{\cW}}{\overline{\cW}}
    \asymp\abs{\nabla U_0}^2.
\]
The choice $2\varepsilon<\eta$ gives
$\overline{\cW}^{\,2}\leq\cW$, while the second inequality in the
choice of $\varepsilon$, together with boundedness of $a\Delta U_0$,
gives
\[
    -2\frac{\calL_\alpha\overline{\cW}}{\overline{\cW}}
    \leq\Psi_\alpha+C_2,
\]
for some $C_2\in\R$.  Finally, $\kappa=1$ is bounded and
$1/\Psi_\alpha(x)\to0$ as $\abs{x}\to\infty$.  All parts of
Assumption~\ref{ass:unconstrained-ldp} follow.
\end{proof}

\begin{remark}
The sufficient condition is stated in terms of the differentiable anchor
$U_0$, rather than the possibly non-differentiable target potential $U$.
Even when $U$ is differentiable, dissipativity of $U$ and boundedness of
$a=e^{U-U_0}$ alone do not control
$\nabla U_0=\nabla U-\nabla\log a$.  If one instead assumes a suitable
sublinear bound on $\nabla\log a$, dissipativity of $U$ transfers to
\eqref{eq:U0-strong-dissipativity}.
\end{remark}

For any measurable weight $w:\R^d\to[1,\infty)$, let
\begin{equation*}
    B^\infty_w
    :=\left\{f\text{ measurable}:
      \sup_{x\in\mathbb{R}^{d}}\frac{|f(x)|}{w(x)}<\infty\right\}.
\end{equation*}
Let $\tau^\kappa$ denote the topology of convergence against all functions in
$B^\infty_\kappa$, and let
\begin{equation*}
    \cP_\kappa(\R^d)
    :=\left\{\nu\in\cP(\R^d):\int_{\mathbb{R}^{d}}\kappa\,d\nu<\infty\right\}.
\end{equation*}

Let $\calD(\calL_\alpha)$ denote the extended generator domain on
$B^\infty_{\cW}$: a function $u\in B^\infty_{\cW}$ belongs to this domain if
there is a measurable $\phi$, denoted by $\calL_\alpha u$, such that
\[
    \int_0^tP_s|\phi|(x)\,ds<\infty,
    \qquad
    P_tu(x)=u(x)+\int_0^tP_s\phi(x)\,ds,
\]
for every $x\in\R^d$ and $t\geq0$.  Following
\cite[Proposition~3]{ferre2020large}, define the positive generalized domain
\begin{equation}\label{eq:unconstrained-positive-domain}
    \calD^+(\calL_\alpha)
    :=\left\{u\in\calD(\calL_\alpha)\cap C^0(\R^d):
      u>0,\ -\frac{\calL_\alpha u}{u}\in B^\infty_\kappa\right\}.
\end{equation}
With the weighted topology and the appropriate positive generator domain now
fixed, the empirical-measure LDP and the irreversible comparison can be stated
without compactness.

\begin{theorem}[Unconstrained LDP and comparison]\label{thm:unconstrained-ldp}
Suppose Assumption~\ref{ass:unconstrained-ldp} holds.  Then $(L_t)_{t\geq0}$ satisfies an LDP on $\cP_\kappa(\R^d)$ equipped with $\tau^\kappa$ with the rate function as the
Donsker--Varadhan functional \begin{equation}\label{eq:DV:unconstrained}
    \cI_\alpha(\nu)
    =-\inf_{u\in\calD^+(\calL_\alpha)}
      \int_{\mathbb{R}^{d}}\frac{\calL_\alpha u}{u}\,d\nu.
\end{equation}
For every regular
$\nu=e^v\pi\in\cP_\kappa(\R^d)$,
\begin{equation}\label{eq:unconstrained-rate}
    \cI_\alpha(\nu)
    =\frac14\int_{\R^d}a\abs{\nabla v}^2d\nu
     +\frac{\alpha^2}{4}\norm{\cA v}_{H^{-1}_a(\nu)}^2
    \geq\cI_0(\nu).
\end{equation}
\end{theorem}

\begin{proof}
The LDP and the Donsker--Varadhan representation follow from the weighted-topology
theorem under Assumption~\ref{ass:unconstrained-ldp}; see
\cite[Theorem~1 and Proposition~3]{ferre2020large} and
\cite{Yao2025LDP}.  In particular, Proposition~3 gives precisely the
generalized domain \eqref{eq:unconstrained-positive-domain}, so no boundary
condition is imposed on $\R^d$.  The symmetric--anti-symmetric decomposition is
the variable-diffusivity specialization of
\cite[Theorem~2]{ferre2020large}: its symmetric carré du champ is
$a\abs{\nabla f}^2$ and its anti-symmetric part is $\alpha\cA$.
For the regular measures stated in the theorem, the integration-by-parts
calculation in the proof of Theorem~\ref{thm:rate-decomposition} therefore
gives \eqref{eq:unconstrained-rate}.
\end{proof}

The canonical perturbation has the advantage that it frequently inherits the Lyapunov
structure of the reversible anchored diffusion.

\begin{proposition}[Transfer of Lyapunov functions]\label{prop:lyapunov-transfer}
Choose $\psi=e^{-U_0}$, so that
\begin{equation}\label{eq:A-canonical}
    \cA=-aJ(x)\nabla U_0\cdot\nabla.
\end{equation}
If $\cW=F(U_0)$ for a differentiable function $F$, then
\begin{equation}\label{eq:A-W-zero}
    \cA\cW=0,
    \qquad
    \calL_\alpha\cW=\cS\cW.
\end{equation}
Consequently, every Witten--Lyapunov estimate based on such a $\cW$ is uniform in
$\alpha$.
\end{proposition}

\begin{proof}
The chain rule gives
\[
    \nabla\cW=F'(U_0)\nabla U_0.
\]
Using \eqref{eq:A-canonical} and the identity
$z^\top J(x)z=0$ for every skew-symmetric matrix $J(x)$, we obtain
\[
    \cA\cW
    =-aF'(U_0)(J(x)\nabla U_0)\cdot\nabla U_0
    =0.
\]
Since $\calL_\alpha=\cS+\alpha\cA$, it follows that
$\calL_\alpha\cW=\cS\cW$ for every $\alpha$.  More precisely,
\[
    \Psi_\alpha
    :=-\frac{\calL_\alpha\cW}{\cW}
    =-\frac{\cS\cW}{\cW}
    =:\Psi_0.
\]
Hence $\Psi_\alpha$ and $\Psi_0$ have identical level sets.  If the auxiliary
function also has the form
$\overline{\cW}=\overline F(U_0)$, the same calculation gives
\[
    -\frac{\calL_\alpha\overline{\cW}}{\overline{\cW}}
    =-\frac{\cS\overline{\cW}}{\overline{\cW}}.
\]
Consequently, the compact-level-set condition
\eqref{eq:witten-lyapunov} and the three inequalities in
\eqref{eq:witten-auxiliary} hold for $\calL_\alpha$ if and only if they hold
for $\cS$, with the same constants $C_1$ and $C_2$.
\end{proof}

\section{Asymptotic Variance Reduction}\label{sec:variance}

In this section, we further study the acceleration of non-reversible anchored Langevin dynamics (NALD)
and non-reversible reflected anchored Langevin dynamics (NRALD) through the lens of
asymptotic variance reduction. 
Let $L^2_0(\pi)$ be the mean-zero subspace of $L^2(\pi)$.  We assume that
there exists a constant $C_a\in(0,\infty)$ such that the weighted
Poincar\'e inequality
\begin{equation}\label{eq:weighted-poincare}
    \Var_\pi(f)
    \leq C_a\int_Ea\abs{\nabla f}^2d\pi.
\end{equation}
On a compact connected domain, \eqref{eq:weighted-poincare} follows from uniform bounds
on $a$ and the usual Poincar\'e inequality.  In the unconstrained setting it is the natural
coercivity condition for anchored dynamics and can hold for polynomial-tailed $\pi$ when
$a$ grows at infinity.

Set
\begin{equation}\label{eq:K-operator}
    \calK:=-\cS
\end{equation}
on $L^2_0(\pi)$.  By \eqref{eq:weighted-poincare}, $\calK$ is strictly positive and has a
bounded inverse.  We assume that
\begin{equation}\label{eq:G-def}
    \calG=\calK^{-1/2}\cA\calK^{-1/2}
\end{equation}
is a densely defined skew-adjoint operator.
A convenient sufficient condition, standard in the theory of non-symmetric
Dirichlet forms \citep[Chapter~I]{ma1992introduction}, is the strong sector bound
\begin{equation}\label{eq:sector-condition}
    \abs{\ip{\cA f}{h}_{L^2(\pi)}}
    \leq C_{\rm sec}
    \left(\int_Ea\abs{\nabla f}^2d\pi\right)^{1/2}
    \left(\int_Ea\abs{\nabla h}^2d\pi\right)^{1/2}
\end{equation}
on a common form core.  It is called a sector condition because it controls the
anti-symmetric bilinear form by the product of the symmetric Dirichlet-energy
norms.  In particular, it makes \eqref{eq:G-def} a bounded skew-adjoint extension.
For a centered observable $g$, let
\begin{equation}\label{eq:ergodic-average}
    \overline g_t=\frac1t\int_0^t g(X_s)\,ds.
\end{equation}
When the Poisson equation
$-\calL_\alpha\phi=g$ admits a solution in the form domain, the standard
Poisson-equation martingale approximation
\citep{duncan2016variance,pavliotis2014stochastic} gives the central limit theorem
\begin{equation}\label{eq:clt}
    \sqrt t\,\overline g_t\ \Longrightarrow\ \mathcal{N}\left(0,\sigma^2_{g,\alpha}\right),
\end{equation}
in  distribution as $t\rightarrow\infty$, where
\begin{equation}\label{eq:green-kubo}
    \sigma^2_{g,\alpha}
    =2\ip{g}{(-\calL_\alpha)^{-1}g}_{L^2(\pi)}.
\end{equation}

The Green--Kubo formula
\citep[Section~3.2, Eq.~(28)]{duncan2016variance}
reduces the variance comparison to an operator
inequality for the resolvent.  Skew-adjointness then makes the effect of the
circulation explicit and identifies exactly when the improvement is strict.

\begin{theorem}[Asymptotic variance reduction]\label{thm:variance-reduction}
Assume \eqref{eq:weighted-poincare}, the skew-adjointness of \eqref{eq:G-def}, and the
validity of \eqref{eq:clt}.  Then, for every $g\in L^2_0(\pi)$,
\begin{equation}\label{eq:variance-resolvent}
    \sigma^2_{g,\alpha}
    =2\ip{\calK^{-1/2}g}
    {\left(I_{d}-\alpha^2\calG^2\right)^{-1}\calK^{-1/2}g}_{L^2(\pi)}
    \leq
    2\ip{g}{\calK^{-1}g}_{L^2(\pi)}
    =\sigma^2_{g,0}.
\end{equation}
If $\alpha\neq0$, the inequality is strict precisely when
\begin{equation}\label{eq:strict-variance}
    \calG\calK^{-1/2}g\neq0.
\end{equation}
When $\calK^{-1}g\in\calD(\cA)$, this is equivalent to
$\cA\calK^{-1}g\neq0$.
\end{theorem}

\begin{proof}
Since $-\calL_\alpha=\calK-\alpha\cA$,
\begin{equation}\label{eq:factor-resolvent}
    (-\calL_\alpha)^{-1}
    =\calK^{-1/2}(I_{d}-\alpha\calG)^{-1}\calK^{-1/2}.
\end{equation}
Because $\calG^*=-\calG$, the symmetric part of the middle resolvent is
\begin{equation}\label{eq:symmetric-resolvent}
    \frac12\left[(I_{d}-\alpha\calG)^{-1}+(I_{d}+\alpha\calG)^{-1}\right]
    =\left(I_{d}-\alpha^2\calG^2\right)^{-1}.
\end{equation}
The operator $-\calG^2$ is nonnegative, and hence
$\left(I_{d}-\alpha^2\calG^2\right)^{-1}\preceq I_{d}$.  Equations
\eqref{eq:green-kubo}--\eqref{eq:symmetric-resolvent} give
\eqref{eq:variance-resolvent}.  For $\alpha\neq0$, equality holds exactly when
$\calK^{-1/2}g\in\ker\calG$, which proves \eqref{eq:strict-variance}.
\end{proof}

On a compact reflected domain, the abstract coercivity and sector hypotheses
can be verified directly from the ordinary Poincar\'e inequality and bounded
coefficients.  The next corollary therefore turns the operator comparison into
an unconditional statement under Assumption~\ref{ass:compact}.

\begin{corollary}[Compact reflected variance reduction]
\label{cor:compact-variance}
Under Assumption~\ref{ass:compact}, the weighted Poincar\'e and strong sector conditions
hold.  Consequently, the stationary reflected diffusion satisfies the central limit
theorem and \eqref{eq:variance-resolvent} for every $g\in L^2_0(\pi)$.  In particular,
its asymptotic variance cannot increase when the tangential circulation is added.
\end{corollary}

\begin{proof}
Let $C_P$ be the ordinary Poincar\'e constant of $(K,\pi)$ and let
$a_{\min}:=\inf_{x\in K}a(x)>0$.  Then
\begin{equation}\label{eq:compact-weighted-poincare-proof}
    \Var_\pi(f)
    \leq C_P\int_K\abs{\nabla f}^2d\pi
    \leq\frac{C_P}{a_{\min}}
         \int_Ka\abs{\nabla f}^2d\pi,
\end{equation}
which proves \eqref{eq:weighted-poincare}.  Since
$\int_K\cA f\,d\pi=0$, we also have
\begin{align*}
    \abs{\ip{\cA f}{h}_{L^2(\pi)}}
    &=\abs{\int_K(c\cdot\nabla f)(h-\pi(h))\,d\pi}\\
    &\leq
    \norm{\frac{\abs{c}^2}{a}}_{L^\infty(K)}^{1/2}
    \left(\int_Ka\abs{\nabla f}^2d\pi\right)^{1/2}
    \norm{h-\pi(h)}_{L^2(\pi)}\\
    &\leq
    \left(\frac{C_P}{a_{\min}}
          \norm{\frac{\abs{c}^2}{a}}_{L^\infty(K)}\right)^{1/2}
    \left(\int_Ka\abs{\nabla f}^2d\pi\right)^{1/2}
    \left(\int_Ka\abs{\nabla h}^2d\pi\right)^{1/2}.
\end{align*}
This is \eqref{eq:sector-condition}.

For $g\in L^2_0(\pi)$, coercivity and the sector bound give a unique weak
Poisson solution $u=(-\calL_\alpha)^{-1}g$ in the mean-zero form domain.
It\^o's formula, the Neumann condition, and stationarity yield
\begin{equation}\label{eq:compact-martingale-decomposition}
    \int_0^t g(X_s)\,ds
    =u(X_0)-u(X_t)+M_t,
    \qquad
    \langle M\rangle_t
    =2\int_0^t a(X_s)\abs{\nabla u(X_s)}^2ds.
\end{equation}
The boundary term vanishes because $\partial_nu=0$.  The endpoint term in
\eqref{eq:compact-martingale-decomposition}, divided by $\sqrt t$, converges
to zero in $L^2$.  The ergodic theorem
\citep[Chapter~4]{pavliotis2014stochastic} is used to identify the limiting
quadratic variation:
\[
    \frac{\langle M\rangle_t}{t}
    =\frac2t\int_0^t a(X_s)\abs{\nabla u(X_s)}^2\,ds
    \longrightarrow
    2\int_Ka\abs{\nabla u}^2\,d\pi,
\]
almost surely as $t\rightarrow\infty$.  The continuous-martingale central limit theorem
\citep[Theorem~7.1.4]{ethier1986markov} then yields that
\[
    \frac{M_t}{\sqrt t}
    \Longrightarrow
    \mathcal{N}\left(0,2\int_Ka\abs{\nabla u}^2d\pi\right),
\]
in distribution as $t\rightarrow\infty$.

Finally, anti-symmetry of $\cA$ gives
\[
    \int_Ka\abs{\nabla u}^2d\pi
    =\ip{u}{-\calL_\alpha u}_{L^2(\pi)}
    =\ip{g}{u}_{L^2(\pi)}.
\]
Thus \eqref{eq:clt} and \eqref{eq:green-kubo} hold; see also
\cite{duncan2016variance,wang2026constrained}.  The conclusion follows from
Theorem~\ref{thm:variance-reduction}.
\end{proof}

The preceding compact argument depends only on a divergence-free tangential
current, not on a constant skew matrix.  Therefore, it applies directly to the
state-dependent stream construction introduced in
Proposition~\ref{prop:state-dependent-current}.

\begin{corollary}[Compact acceleration with a state-dependent tensor]
\label{cor:state-dependent-compact}
Assume the domain, coefficient regularity, ellipticity, and well-posedness requirements of
Assumption~\ref{ass:compact}, and set
\begin{equation*}
    c=e^UJ(x)\nabla\psi,
\end{equation*}
where $J\in C^1(\overline K;\R^{d\times d})$ satisfies
\eqref{eq:state-dependent-J-assumptions}.  Suppose either condition of
Lemma~\ref{lem:stream-boundary} holds, so that Assumption~\ref{ass:compact} is satisfied.
Then the reflected diffusion has invariant
measure $\pi$, its empirical-measure rate function satisfies
\eqref{eq:rate-decomposition}--\eqref{eq:rate-order}, and every
$g\in L^2_0(\pi)$ satisfies the central limit theorem:
\begin{equation*}
    \sqrt t\,\overline g_t
    \Longrightarrow \mathcal{N}\left(0,\sigma^2_{g,\alpha}\right),
\end{equation*}
in distribution as $t\rightarrow\infty$, where
\begin{equation*}
\sigma^2_{g,\alpha}\leq\sigma^2_{g,0}.
\end{equation*}
The strictness criteria are respectively $\cA v\neq0$ in $H^{-1}_a(\nu)$ and
$\calG\calK^{-1/2}g\neq0$.
\end{corollary}

\begin{proof}
Proposition~\ref{prop:state-dependent-current} gives
\[
    \nabla\cdot(\pi c)
    =Z^{-1}\nabla\cdot(J\nabla\psi)=0.
\]
Indeed, the term involving $\nabla\cdot J$ vanishes by
\eqref{eq:state-dependent-J-assumptions}, and the contraction of the
skew-symmetric matrix $J$ with the symmetric Hessian $\nabla^2\psi$ is zero.
Lemma~\ref{lem:stream-boundary} gives $c\cdot n=0$ on $\partial K$.
Consequently,
\[
    \int_K\cA f\,d\pi
    =\int_Kc\cdot\nabla f\,d\pi=0,
    \qquad
    \ip{f}{\cA h}_{L^2(\pi)}
    =-\ip{\cA f}{h}_{L^2(\pi)},
\]
with no boundary contribution.  Thus Theorem~\ref{thm:invariance} gives
invariance of $\pi$, and Assumption~\ref{ass:compact} permits application of
Theorem~\ref{thm:rate-decomposition}.

Moreover, compactness and coefficient regularity imply
$a_{\min}>0$ and $\|c\|_{L^\infty(K)}<\infty$.  The calculation in the proof
of Corollary~\ref{cor:compact-variance} therefore proves the strong sector
bound with
\[
    C_{\rm sec}^2
    \leq\frac{C_P}{a_{\min}}
         \norm{\frac{\abs{c}^2}{a}}_{L^\infty(K)}.
\]
Corollary~\ref{cor:compact-variance} now gives the central limit theorem and
variance comparison, while the strictness statements follow respectively
from \eqref{eq:rate-decomposition} and \eqref{eq:strict-variance}.
\end{proof}

If $\calG$ is compact, Theorem~\ref{thm:variance-reduction} yields a more explicit
spectral formula.  Suppose $\calG e_j=\mathrm{i}\lambda_j e_j$, and decompose
\begin{equation*}
    \calK^{-1/2}g=h_0+\sum_{j\geq1}h_j e_j,
    \qquad h_0\in\ker\calG.
\end{equation*}
Then
\begin{equation}\label{eq:spectral-variance}
    \sigma^2_{g,\alpha}
    =2\norm{h_0}_{L^2(\pi)}^2
     +2\sum_{j\geq1}\frac{\abs{h_j}^2}{1+\alpha^2\lambda_j^2}.
\end{equation}
Thus, increasing $\abs{\alpha}$ damps precisely those components of the reversible
Poisson solution which are moved by the circulation.

\begin{remark}
Continuous-time variance reduction does not automatically imply a reduction per
unit of computation.  Large $\abs{\alpha}$ makes an explicit discretization stiff and may
require a smaller stepsize. Therefore, for the numerical comparisons in Section~\ref{sec:numerical}, we use a fixed number of
evaluations of $U$, $\nabla U_0$, and $\nabla\psi$, and report both discretization
bias and effective sample size.
\end{remark}

\begin{remark}
A constant shift $U_0\mapsto U_0+C$ multiplies $a$ and hence the entire
reversible infinitesimal generator by $e^{-C}$.  Both the level-two rate function and the inverse
asymptotic variance scale with this clock speed.  To compare different anchors, one should
fix a normalization such as $\int_{E} a\,d\pi=1$, a prescribed average quadratic variation,
or an explicit computational budget.  The comparisons in
Theorems~\ref{thm:rate-decomposition} and \ref{thm:variance-reduction} avoid this
ambiguity because $a$ is held fixed while only the anti-symmetric strength $\alpha$ varies.
\end{remark}

\section{Explicit Constructions and Examples}\label{sec:examples}

\subsection{The anchored heavy-tailed class}\label{sec:anchored-heavy-tail}

Throughout this subsection, we specialize the general matrix field to the
constant-matrix case: $J(x)\equiv J$ for a fixed nonzero skew-symmetric matrix $J$.

The general heavy-tailed family needed here is already present in
\cite[Theorem~3]{gurbuzbalaban2025anchored}.  We record it in our notation and isolate
the extra point needed for the empirical-measure LDP.  Let $q:\R^d\to[1,\infty)$ be
sufficiently smooth and suppose, for some $\beta>0$, that
\begin{equation}\label{eq:q-integrability}
    Z_q:=\int_{\R^d}\left(q(x)\right)^{-(\beta+1)}\,dx<\infty.
\end{equation}
Assume that there are $c_0,c_1>0$ and $r>-1$ such that
\begin{equation}\label{eq:q-heavy-tail-drift}
    d q(x)-\beta x\cdot\nabla q(x)
    \leq-c_0\abs{x}^{2+r}+c_1,
    \qquad x\in\R^d.
\end{equation}
The choice
\begin{equation}\label{eq:q-heavy-tail-potentials}
    U_0=\beta\log q,
    \qquad U=(\beta+1)\log q
\end{equation}
gives $a=q$.  Hence, the canonical nonreversible anchored diffusion is
\begin{equation}\label{eq:q-heavy-tail-sde}
    dX_t=-\beta(I_{d}+\alpha J)\nabla q(X_t)\,dt
          +\sqrt{2q(X_t)}\,dW_t,
\end{equation}
and its candidate invariant density is $Z_q^{-1}q^{-(\beta+1)}$.

The following proposition explains exactly which conclusions follow directly from
\eqref{eq:q-integrability} and \eqref{eq:q-heavy-tail-drift}.
The auxiliary Witten--Lyapunov requirements
\eqref{eq:witten-lyapunov}--\eqref{eq:witten-auxiliary} in
Assumption~\ref{ass:unconstrained-ldp}
are checked separately for the Student family below.

\begin{proposition}[Heavy-tail verification]\label{prop:q-heavy-tail}
Suppose the local regularity conditions needed for the diffusion and
\eqref{eq:q-integrability}--\eqref{eq:q-heavy-tail-drift} hold and the
anti-symmetric flow is tangent to the Euclidean level sets,
\begin{equation}\label{eq:q-tangency}
    x\cdot J\nabla q(x)=0,
    \qquad x\in\R^d.
\end{equation}
Then the diffusion \eqref{eq:q-heavy-tail-sde} is nonexplosive and has the unique
invariant probability measure
\begin{equation}\label{eq:q-heavy-tail-target}
    \pi_q(dx)=Z_q^{-1}\left(q(x)\right)^{-(\beta+1)}\,dx.
\end{equation}
Moreover:
\begin{enumerate}[label=\textup{(\roman*)},leftmargin=2.5em]
    \item if $r\geq0$, the symmetric anchored infinitesimal generator satisfies a weighted Poincar\'e
    inequality
    \begin{equation}\label{eq:q-weighted-poincare}
        \Var_{\pi_q}(f)
        \leq C_q\int_{\R^d}q\abs{\nabla f}^2d\pi_q;
    \end{equation}
    \item if $r>0$, the principal coercive condition in
    Assumption~\ref{ass:unconstrained-ldp} holds with
    $\cW(x)=1+\abs{x}^2$.  It permits $\kappa=1$, and more generally every admissible
    weight $\kappa=o(1+\abs{x}^r)$.  If the auxiliary condition
    \eqref{eq:witten-auxiliary} and the coefficient assumptions are also satisfied,
    Theorem~\ref{thm:unconstrained-ldp} gives the rate-function acceleration.
\end{enumerate}
\end{proposition}

\begin{proof}
The substitution \eqref{eq:q-heavy-tail-potentials} in the reversible anchored
infinitesimal generator gives
\begin{equation*}
    \cS=q\Delta-\beta\nabla q\cdot\nabla,
\end{equation*}
and the canonical perturbation is
$\cA=-\beta J\nabla q\cdot\nabla$.  Thus
$\calL_\alpha=\cS+\alpha\cA$ is precisely the infinitesimal generator of
\eqref{eq:q-heavy-tail-sde}.  The divergence-form calculation in
Theorem~\ref{thm:invariance} proves invariance of \eqref{eq:q-heavy-tail-target}; local
ellipticity, irreducibility, and the Lyapunov estimate below give nonexplosion and
uniqueness.

For $\cW(x)=1+\abs{x}^2$, condition \eqref{eq:q-tangency} yields
\begin{align}
    \calL_\alpha\cW
    &=2\left(dq-\beta x\cdot\nabla q
              -\alpha\beta x\cdot J\nabla q\right)\notag\\
    &=2\left(dq-\beta x\cdot\nabla q\right)
      \leq-2c_0\abs{x}^{2+r}+2c_1.\label{eq:q-W-drift}
\end{align}
For $r\geq0$, the right-hand side is bounded above by
$-\lambda\cW+b\mathbf 1_C$ for a compact set $C$.  The standard
Lyapunov--Poincar\'e criterion
\citep[Theorem~1.4 and Section~2]{bakry2008simple}
can be seen directly from the ground-state
inequality
\begin{equation}\label{eq:ground-state-inequality}
    \int_{\R^d}
      \left(-\frac{\cS\cW}{\cW}\right)f^2d\pi_q
    \leq
    \int_{\R^d}q\abs{\nabla f}^2d\pi_q.
\end{equation}
Indeed, integration by parts gives
\[
    \int_{\mathbb{R}^{d}} q\nabla\cW\cdot\nabla\left(\frac{f^2}{\cW}\right)d\pi_q
    =
    \int_{\mathbb{R}^{d}} q\left(
       2\frac{f}{\cW}\nabla f\cdot\nabla\cW
       -\frac{f^2}{\cW^2}\abs{\nabla\cW}^2\right)d\pi_q,
\]
and the right-hand side is bounded above by the Dirichlet energy in
\eqref{eq:ground-state-inequality}.  Combining this inequality with
$\cS\cW\leq-\lambda\cW+b\mathbf1_C$ controls
$\int_{C^c}f^2d\pi_q$ by the Dirichlet energy and the $L^2$ mass on $C$.
The local Poincar\'e inequality on $C$, together with $\pi_q(f)=0$, controls
the latter term; the usual enlargement-and-absorption argument therefore gives
\eqref{eq:q-weighted-poincare}.  See \cite{bakry2008rate} for this
Lyapunov--Poincar\'e patching argument.

If $r>0$, \eqref{eq:q-W-drift} also gives
\begin{equation}\label{eq:q-Psi-bound}
    \Psi_\alpha(x)
    =-\frac{\calL_\alpha\cW(x)}{\cW(x)}
    \geq
    \frac{2c_0\abs{x}^{2+r}-2c_1}{1+\abs{x}^2}
    \longrightarrow\infty
    \qquad\text{as }\abs{x}\to\infty.
\end{equation}
Thus, $\Psi_\alpha$ has compact level sets and
$\kappa/\Psi_\alpha\to0$ for each $\kappa=o(1+\abs{x}^r)$, proving the claimed
principal coercive estimate.
\end{proof}

Condition \eqref{eq:q-tangency} is automatic when $q$ is radial, since a
skew-symmetric matrix satisfies $\langle x,Jx\rangle=0$.  For an ellipsoidal $q$ it holds, after
centering, whenever $J$ is compatible with the shape matrix; alternatively, one may use
a stream potential whose current is tangent to the selected Lyapunov level sets.

\begin{corollary}[A Student-$t$ family in the LDP regime]\label{cor:student-ldp}
Let $\vartheta>2$ be the degrees of freedom and choose
\begin{equation}\label{eq:student-zeta-range}
    1<\zeta<\frac{\vartheta}{2},
    \qquad
    s(x)=1+\frac{\abs{x}^2}{\vartheta},
    \qquad
    q_\zeta(x)=(s(x))^\zeta,
    \qquad
    \beta=\frac{d+\vartheta}{2\zeta}-1.
\end{equation}
Then \eqref{eq:q-heavy-tail-sde} becomes
\begin{equation}\label{eq:student-zeta-sde}
    dX_t=-\frac{2\beta\zeta}{\vartheta}\left(s(X_t)\right)^{\zeta-1}
          (I_{d}+\alpha J)X_t\,dt
          +\sqrt{2}\left(s(X_t)\right)^{\zeta/2}\,dW_t,
\end{equation}
and has the centered isotropic Student-$t$ density
\begin{equation}\label{eq:student-zeta-target}
    \pi_\vartheta(x)\propto
    \left(1+\frac{\abs{x}^2}{\vartheta}\right)^{-(d+\vartheta)/2}.
\end{equation}
For every constant skew-symmetric $J$, the empirical measures of
\eqref{eq:student-zeta-sde} satisfy Theorem~\ref{thm:unconstrained-ldp}, and the
associated symmetric infinitesimal generator satisfies \eqref{eq:q-weighted-poincare}.
\end{corollary}

\begin{proof}
The identity $\zeta(\beta+1)=(d+\vartheta)/2$ proves
\eqref{eq:student-zeta-target}.  A direct calculation gives
\begin{equation}\label{eq:student-q-drift}
    dq_\zeta-\beta x\cdot\nabla q_\zeta
    =s^{\zeta-1}
      \left[(2\zeta-\vartheta)s+d+\vartheta-2\zeta\right].
\end{equation}
Since $2\zeta-\vartheta<0$, this is bounded above by
$-c_0\abs{x}^{2\zeta}+c_1$.  Hence
\eqref{eq:q-heavy-tail-drift} holds with $r=2\zeta-2>0$.  Moreover,
$\nabla q_\zeta$ is parallel to $x$, so \eqref{eq:q-tangency} holds for every
skew-symmetric $J$.  Proposition~\ref{prop:q-heavy-tail} gives
\eqref{eq:q-weighted-poincare} and the principal Lyapunov estimate.

It remains to verify all parts of Assumption~\ref{ass:unconstrained-ldp}.  The drift and
diffusion coefficients in \eqref{eq:student-zeta-sde} are smooth and have polynomially
growing derivatives.  Since $s^{\zeta/2}>0$, ellipticity gives the H\"ormander condition;
invertibility of the diffusion coefficient gives controllability by following any smooth
path between two prescribed points.  To verify
\eqref{eq:witten-lyapunov}--\eqref{eq:witten-auxiliary} in
Assumption~\ref{ass:unconstrained-ldp}, set
\begin{equation*}
    V(x)=1+\log s(x).
\end{equation*}
Then $V$ and $\abs{\sqrt{2q_\zeta}\nabla V}$ have compact level sets because
\begin{equation}\label{eq:student-gradient-V}
    q_\zeta\abs{\nabla V}^2
    =\frac{4}{\vartheta}s^{\zeta-2}(s-1)
    \asymp s^{\zeta-1}.
\end{equation}
Moreover $\cA V=0$, and, for every $\delta\in(0,1)$,
\begin{align}
    -\calL_\alpha V-\delta q_\zeta\abs{\nabla V}^2
    &=q_\zeta\left[(\beta\zeta-\delta)\abs{\nabla V}^2
                    -\Delta V\right]\notag\\
    &=\frac{s^{\zeta-2}}{\vartheta}
      \left[4(\beta\zeta+1-\delta)(s-1)-2d s\right]
      \asymp s^{\zeta-1}.
      \label{eq:student-nonlinear-lyapunov}
\end{align}
Indeed, the coefficient of $s^{\zeta-1}$ in
\eqref{eq:student-nonlinear-lyapunov}
is
\begin{equation*}
    4(\beta\zeta+1-\delta)-2d
    =2(\vartheta-2\zeta+2-2\delta)>0.
\end{equation*}
The nonlinear criterion of \cite[Proposition~2.9]{ferre2020large}, applied with
$\cW=e^{\eta V}$ and $\overline{\cW}=e^{\varepsilon V}$ for sufficiently small
$0<\varepsilon<\eta/2$, now gives \eqref{eq:witten-lyapunov}--
\eqref{eq:witten-auxiliary}, with $\Psi_\alpha\asymp s^{\zeta-1}$.  The bounded choice
$\kappa=1$ falls under the bounded-weight alternative.  Thus all parts of
Assumption~\ref{ass:unconstrained-ldp} hold, and
Theorem~\ref{thm:unconstrained-ldp} completes the proof.
\end{proof}

For the superquadratic Student clocks, the remaining input for the variance
theorem is a strong sector estimate.  A pointwise coefficient bound is
unavailable, but the current is purely angular, so the spherical Poincar\'e
inequality supplies the required form bound.

\begin{proposition}[Strong sector bound for Student clocks]
\label{prop:student-sector-superquadratic}
Under the hypotheses of Corollary~\ref{cor:student-ldp}, and with $d\geq2$, let
\begin{equation*}
    \cE_\zeta(f,f)
    :=\int_{\R^d}q_\zeta\abs{\nabla f}^2d\pi_\vartheta.
\end{equation*}
Then, for all $f,h\in C_c^\infty(\R^d)$,
\begin{equation}\label{eq:student-strong-sector}
    \abs{\ip{\cA f}{h}_{L^2(\pi_\vartheta)}}
    \leq
    \frac{2\beta\zeta\norm{J}_{\rm op}}{\sqrt{d-1}}
    \left(\cE_\zeta(f,f)\right)^{1/2}\left(\cE_\zeta(h,h)\right)^{1/2}.
\end{equation}
Consequently, the canonical form realization of
\begin{equation*}
    \calG=\calK^{-1/2}\cA\calK^{-1/2}
\end{equation*}
has a unique bounded skew-adjoint extension to $L^2_0(\pi_\vartheta)$, with the operator
norm satisfying
\begin{equation*}
    \norm{\calG}_{\rm op}
    \leq\frac{2\beta\zeta\norm{J}_{\rm op}}{\sqrt{d-1}}.
\end{equation*}
\end{proposition}

\begin{proof}
Write $x=r\theta$ with $r>0$ and $\theta\in\mathbb S^{d-1}$, and let $d\omega$ be
the surface measure on $\mathbb S^{d-1}$.  Since $J\theta$ is tangent to the sphere,
\begin{equation}\label{eq:J-angular-derivative}
    Jx\cdot\nabla f(r\theta)=J\theta\cdot\nabla_{\mathbb S}f(r\theta)
    =:D_Jf(r,\theta).
\end{equation}
The vector field $J\theta$ is divergence-free on the sphere: indeed,
$\operatorname{div}_{\mathbb S}(J\theta)
=\operatorname{tr}J-\theta^\top J\theta=0$.  Hence
\begin{equation*}
    \int_{\mathbb S^{d-1}}D_Jf(r,\theta)\,d\omega(\theta)=0.
\end{equation*}
If $\overline h(r)$ denotes the spherical mean of $h(r,\cdot)$, the sharp Poincar\'e
inequality on $\mathbb S^{d-1}$ is
\begin{equation}\label{eq:sphere-poincare}
    \int_{\mathbb S^{d-1}}\abs{h-\overline h}^2d\omega
    \leq\frac{1}{d-1}
    \int_{\mathbb S^{d-1}}\abs{\nabla_{\mathbb S}h}^2d\omega.
\end{equation}
Here $d-1$ is the first nonzero eigenvalue of
$-\Delta_{\mathbb S^{d-1}}$; see, for example, \cite{Bakry2013}.
Combining \eqref{eq:sphere-poincare} with
$\abs{D_Jf}\leq\norm{J}_{\rm op}\abs{\nabla_{\mathbb S}f}$ gives
\begin{align}
    \abs{\int_{\mathbb S^{d-1}}D_Jf\,h\,d\omega}
    =\abs{\int_{\mathbb S^{d-1}}D_Jf\,(h-\overline h)\,d\omega}
    \leq\frac{\norm{J}_{\rm op}}{\sqrt{d-1}}
      \norm{\nabla_{\mathbb S}f}_{L^2(d\omega)}
      \norm{\nabla_{\mathbb S}h}_{L^2(d\omega)}.
      \label{eq:spherical-sector}
\end{align}

Let $p_\vartheta(r)$ denote the radial density factor of $\pi_\vartheta$, so that
$d\pi_\vartheta=p_\vartheta(r)r^{d-1}drd\omega$.  Since
\begin{equation*}
    \cA=-\frac{2\beta\zeta}{\vartheta}s^{\zeta-1}D_J,
\end{equation*}
and
\begin{equation}\label{eq:radial-weight-comparison}
    s^{\zeta-1}
    =\frac{r^2}{s}\frac{s^\zeta}{r^2}
    \leq\vartheta\frac{s^\zeta}{r^2},
\end{equation}
integration of \eqref{eq:spherical-sector} in $r$, followed by Cauchy--Schwarz inequality, yields
\begin{align*}
    \abs{\ip{\cA f}{h}_{L^2(\pi_\vartheta)}}
    &\leq\frac{2\beta\zeta\norm{J}_{\rm op}}{\sqrt{d-1}}
      \left(\int_0^\infty p_\vartheta r^{d-1}
        \frac{s^\zeta}{r^2}
        \norm{\nabla_{\mathbb S}f}_{L^2(d\omega)}^2dr\right)^{1/2}\\
    &\qquad\qquad\qquad\qquad\cdot
      \left(\int_0^\infty p_\vartheta r^{d-1}
        \frac{s^\zeta}{r^2}
        \norm{\nabla_{\mathbb S}h}_{L^2(d\omega)}^2dr\right)^{1/2}\\
    &\leq\frac{2\beta\zeta\norm{J}_{\rm op}}{\sqrt{d-1}}
      \left(\cE_\zeta(f,f)\right)^{1/2}\left(\cE_\zeta(h,h)\right)^{1/2}.
\end{align*}
This proves \eqref{eq:student-strong-sector}.

By the weighted Poincar\'e inequality \eqref{eq:q-weighted-poincare},
$\calK^{-1/2}$ is bounded from
$L^2_0(\pi_\vartheta)$ into the form domain and
\[
\cE_\zeta\left(\calK^{-1/2}u,\calK^{-1/2}u\right)=\norm{u}_{L^2(\pi_\vartheta)}^2.
\]
The anti-symmetric form $(f,h)\mapsto\ip{\cA f}{h}$ therefore extends continuously to
the full form domain.  Pulling it back by $\calK^{-1/2}$ defines $\calG$ as a bounded
operator on all of $L^2_0(\pi_\vartheta)$.  Integration by parts on the compactly
supported core gives
$\ip{\cA f}{h}=-\ip{f}{\cA h}$; density and continuity then give
$\calG^*=-\calG$ after extension.  Uniqueness follows from density of the form core.
\end{proof}

\begin{remark}
For $\zeta>1$, the usual sufficient estimate based on $\abs{c}^2/q_\zeta$ is
unavailable, since
\begin{equation*}
    \frac{\abs{c(x)}^2}{q_\zeta(x)}
    =\left(\frac{2\beta\zeta}{\vartheta}\right)^2
      (s(x))^{\zeta-2}\abs{Jx}^2
    \asymp \abs{x}^{2\zeta-2},
\end{equation*}
along directions not in $\ker J$.  Proposition~\ref{prop:student-sector-superquadratic}
works because the current is purely angular: its spherical mean vanishes, and the
spherical spectral gap controls the remaining angular transport by the angular part of
the Dirichlet energy.  Thus the form-sector estimate remains uniform even though the
corresponding pointwise coefficient is unbounded.
\end{remark}

Combining this angular sector estimate with the weighted Poincar\'e inequality
verifies all operator hypotheses of Theorem~\ref{thm:variance-reduction}.
This yields variance reduction for the same superquadratic Student model for
which the weighted-topology LDP was proved above.

\begin{corollary}[Variance reduction for superquadratic Student clocks]
\label{cor:student-superquadratic-variance}
Under Corollary~\ref{cor:student-ldp}, with $d\geq2$, the stationary diffusion
\eqref{eq:student-zeta-sde} satisfies the central limit theorem and
Theorem~\ref{thm:variance-reduction} for every $g\in L^2_0(\pi_\vartheta)$.  Hence
\begin{equation}\label{eq:student-superquadratic-variance-order}
    \sigma^2_{g,\alpha}\leq\sigma^2_{g,0}.
\end{equation}
If $\alpha\neq0$ and $g_u(x)=u^\top x$ with $Ju\neq0$, then the inequality in
\eqref{eq:student-superquadratic-variance-order} is strict.
\end{corollary}

\begin{proof}
The weighted Poincar\'e inequality follows from
Proposition~\ref{prop:q-heavy-tail}(i), while
Proposition~\ref{prop:student-sector-superquadratic} supplies the strong sector
condition \eqref{eq:student-strong-sector} and skew-adjointness of $\calG$.
Together with \eqref{eq:q-weighted-poincare}, this bound controls the
anti-symmetric form by the symmetric Dirichlet-energy norm and gives the
Poisson-equation martingale approximation \citep{duncan2016variance}.
The approximation therefore gives the stationary central limit
theorem, and
Theorem~\ref{thm:variance-reduction} proves \eqref{eq:student-superquadratic-variance-order}.

For strictness, rotational invariance of $\calK$ and uniqueness of its inverse imply that
\begin{equation}\label{eq:student-linear-poisson}
    \calK^{-1}g_u(x)=\varphi(\abs{x})u^\top x
\end{equation}
for a nonzero radial function $\varphi$.  Since $\langle Jx,x\rangle=0$,
\begin{equation}\label{eq:student-angular-action}
    \cA\calK^{-1}g_u(x)
    =-\frac{2\beta\zeta}{\vartheta}(s(x))^{\zeta-1}
      \varphi(\abs{x})u^\top Jx,
\end{equation}
which is nonzero as a form functional whenever $Ju\neq0$.  Indeed, by
\eqref{eq:student-angular-action}, one may choose a compact radial annulus
on which $\varphi$ is nonzero on a set of positive radial measure and
approximate the right-hand side of \eqref{eq:student-angular-action} by
compactly supported smooth test functions.
Hence $\calG\calK^{-1/2}g_u\neq0$, and the strictness criterion in
Theorem~\ref{thm:variance-reduction} completes the proof.
\end{proof}

\begin{remark}\label{rem:student-borderline}
The original heavy-tail theorem \citep[Theorems~2--3]{gurbuzbalaban2025anchored}
proves invariance for $r>-1$ and exponential ergodicity for $r\geq0$.  These statements
do not by themselves prove the weighted-topology empirical-measure LDP used here.  At
the quadratic-clock endpoint $\zeta=1$, one has $r=0$ and
$-\calL_\alpha\cW/\cW$ remains of constant order at infinity, rather than tending to
infinity.  Thus Proposition~\ref{prop:q-heavy-tail}(i) still supplies the coercivity
needed for the asymptotic-variance comparison, but our Witten--Lyapunov proof of the LDP
requires the strict choice $\zeta>1$.  For $\zeta>1$, Corollaries
\ref{cor:student-ldp} and \ref{cor:student-superquadratic-variance} give both notions of
acceleration for the same Student-$t$ target.  The quadratic-clock example below remains
useful because its linear-observable variance can be evaluated in closed form.
\end{remark}

\subsection{A polynomial-tailed example with explicit variance}

We give an example in which the anchored diffusion has a heavy-tailed invariant measure
but the nonreversible drift and the asymptotic variance of linear observables are explicit.
Let $d\geq2$, let $\beta>d/2$, and set
\begin{equation}\label{eq:student-potentials}
    q(x)=1+\abs{x}^2,
    \qquad
    U_0(x)=\beta\log q(x),
    \qquad
    U(x)=(\beta+1)\log q(x),
\end{equation}
where $J\in\R^{d\times d}$ is any fixed skew-symmetric matrix.
Then $a=q$ and the canonical stream potential $\psi=e^{-U_0}$ gives
\begin{equation}\label{eq:student-sde}
    dX_t=-2\beta(I_{d}+\alpha J)X_t\,dt
          +\sqrt{2\left(1+\abs{X_t}^2\right)}\,dW_t.
\end{equation}
Its invariant density is
\begin{equation}\label{eq:student-density}
    \pi_{\beta,d}(x)\propto\left(1+\abs{x}^2\right)^{-(\beta+1)}.
\end{equation}
Since
\[
    dq-\beta x\cdot\nabla q
    =d-(2\beta-d)\abs{x}^2,
\]
condition \eqref{eq:q-heavy-tail-drift} holds with $r=0$, so
Proposition~\ref{prop:q-heavy-tail}(i) gives the weighted Poincar\'e inequality.  The
sector condition is also automatic in this example.  Indeed,
\begin{equation}\label{eq:student-sector}
    \norm{\cA f}_{L^2(\pi_{\beta,d})}^2
    =4\beta^2\int_{\R^d}\abs{Jx\cdot\nabla f}^2d\pi_{\beta,d}
    \leq4\beta^2\norm{J}_{\rm op}^2
       \int_{\R^d}\left(1+\abs{x}^2\right)
                    \abs{\nabla f}^2d\pi_{\beta,d}.
\end{equation}
Combining \eqref{eq:student-sector} with the weighted Poincar\'e inequality yields
\eqref{eq:sector-condition}.

For the quadratic clock, the drift closes on linear observables.  The
conditional mean and Green--Kubo integral can therefore be evaluated
explicitly, giving a closed-form version of the general variance comparison.

\begin{proposition}[Exact reduction for linear observables]\label{prop:student-variance}
For $g_u(x)=u^\top x$, $u\in\R^d$,
\begin{equation}\label{eq:student-exact-variance}
    \sigma^2_{u,\alpha}
    =\frac{1}{\beta(2\beta-d)}
      u^\top\left(I_{d}-\alpha^2J^2\right)^{-1}u
    \leq\frac{\abs{u}^2}{\beta(2\beta-d)}
    =\sigma^2_{u,0}.
\end{equation}
For $\alpha\neq0$, the inequality is strict if and only if $Ju\neq0$.
If $d$ is even and $J$ is an orthogonal complex structure, so that
$J^2=-I_{d}$, then
\begin{equation}\label{eq:student-even-dimensional-ratio}
    \frac{\sigma^2_{u,\alpha}}{\sigma^2_{u,0}}
    =\frac1{1+\alpha^2}.
\end{equation}
In particular, such a $J$ is obtained by taking a block diagonal matrix with
$d/2$ copies of
\[
    J_2=\begin{pmatrix}0&-1\\1&0\end{pmatrix}.
\]
\end{proposition}

\begin{proof}
Rotational symmetry and a beta-integral give
\begin{equation}\label{eq:student-covariance}
    \int_{\R^d}xx^\top\pi_{\beta,d}(dx)=\frac1{2\beta-d}I_d.
\end{equation}
The conditional mean of \eqref{eq:student-sde} is
\begin{equation}\label{eq:student-conditional-mean}
    \E[X_t\mid X_0=x]=e^{-2\beta(I_{d}+\alpha J)t}x.
\end{equation}
Using \eqref{eq:student-covariance}--\eqref{eq:student-conditional-mean}
in the Green--Kubo formula
\[
    \sigma^2_{g,\alpha}
    =2\int_0^\infty\ip{g}{P_tg}_{L^2(\pi_{\beta,d})}\,dt
    =2\ip{g}{(-\calL_\alpha)^{-1}g}_{L^2(\pi_{\beta,d})},
\]
which is \eqref{eq:green-kubo} (see also
\cite{duncan2016variance,pavliotis2014stochastic}),
we obtain
\begin{align*}
    \sigma^2_{u,\alpha}
    &=2\int_0^\infty
      \E_{\pi_{\beta,d}}[g_u(X_0)g_u(X_t)]\,dt\\
    &=\frac1{\beta(2\beta-d)}
      u^\top(I_{d}-\alpha J)^{-1}u
      =\frac1{\beta(2\beta-d)}
      u^\top\left(I_{d}-\alpha^2J^2\right)^{-1}u.
\end{align*}
To justify the last equality, note that $J^\top=-J$ and hence
\[
    \left[(I_{d}-\alpha J)^{-1}\right]^\top
    =(I_{d}+\alpha J)^{-1}.
\]
The anti-symmetric part of a real matrix does not contribute to its quadratic
form.  
Therefore,
\begin{align*}
    u^\top(I_{d}-\alpha J)^{-1}u
    =
    \frac12u^\top\left[
       (I_{d}-\alpha J)^{-1}+(I_{d}+\alpha J)^{-1}
    \right]u
    =
    u^\top\left(I_{d}-\alpha^2J^2\right)^{-1}u,
\end{align*}
where the second equality follows from
\[
    (I_{d}-\alpha J)^{-1}
    =(I_{d}+\alpha J)\left(I_{d}-\alpha^2J^2\right)^{-1},
    \qquad
    (I_{d}+\alpha J)^{-1}
    =(I_{d}-\alpha J)\left(I_{d}-\alpha^2J^2\right)^{-1}.
\]
Since $-J^2$ is positive semidefinite,
$\left(I_{d}-\alpha^2J^2\right)^{-1}\preceq I_{d}$, with equality on a vector $u$ exactly
when $\alpha Ju=0$.  This proves \eqref{eq:student-exact-variance}; the
special case $J^2=-I_{d}$ gives
\eqref{eq:student-even-dimensional-ratio}.
\end{proof}

\begin{remark}
When $d$ is odd, every skew-symmetric $J$ has a nontrivial kernel.  The formula
\eqref{eq:student-exact-variance} shows that the variance is unchanged in
directions in $\ker J$ and is strictly reduced in every direction $u$ with
$Ju\neq0$.
\end{remark}

\begin{remark}
The reversible diffusion associated directly with the polynomial-tailed density
\eqref{eq:student-density} need not possess a standard spectral gap.  In contrast, the
anchored Dirichlet form contains the growing weight $a(x)=1+\abs{x}^2$, which is the
appropriate setting for a weighted Poincar\'e inequality.  This illustrates why the
coercivity condition \eqref{eq:weighted-poincare}, rather than an unweighted inequality,
is natural here.
\end{remark}

\section{Numerical Experiments}\label{sec:numerical}

The Euler-Maruyama discretization of \eqref{eq:nald} is
\begin{equation}\label{eq:algorithm-unconstrained}
    x_{k+1}=x_k-\eta a(x_k)\nabla U_0(x_k)
    +\eta\alpha e^{U(x_k)}J(x_k)\nabla\psi(x_k)
    +\sqrt{2\eta a(x_k)}\,\xi_{k+1},
\end{equation}
where $\xi_k$ are independent standard Gaussian vectors.  On a convex constrained
domain, the corresponding projected scheme is
\begin{equation}\label{eq:algorithm-constrained}
    x_{k+1}=\Pi_K\left(x_k-\eta a(x_k)\nabla U_0(x_k)
    +\eta\alpha e^{U(x_k)}J(x_k)\nabla\psi(x_k)
    +\sqrt{2\eta a(x_k)}\,\xi_{k+1}\right),
\end{equation}
where
$\Pi_{K}(x):=\arg\min_{y\in\bar{K}}\| y-x\|$.
In our numerical examples, $U$ always takes the form:
\begin{equation}\label{eq:Ufg}
U(x)=f(x)+g(x),
\end{equation}
where $f(x)$ is differentiable but $g(x)$ is not. 
We construct the anchored reference potential as
\begin{equation}\label{eq:U0anchor}
U_0(x)=f(x)+g_{0}(x),
\end{equation}
where $g_{0}(x)$ is a smoothed version of $g(x)$ such as the Gaussian smoothing as considered in \cite{gurbuzbalaban2025anchored}.

\subsection{Unconstrained sampling}

\subsubsection{A heavy-tailed target with a non-smooth regularizer}
\label{sec:heavy}

We conduct numerical experiments on a target that is both heavy-tailed and
non-smooth, built from the anchored heavy-tailed family of
Section~\ref{sec:examples} in the quadratic-clock case 
and the minimax concave penalty (MCP) regularizer (see e.g. \cite{Zhang10}). Let
\begin{equation}
\label{eq:student-target}
\pi(x)\propto \big(q(x)\big)^{-(\beta+1)}\,e^{-g(x)},\qquad
q(x):=1+\frac{1}{\nu}\,x^\top\Sigma^{-1}x,\qquad
\beta:=\frac{d+\nu}{2}-1 ,
\end{equation}
where $q^{-(\beta+1)}$ is, up to normalization, the density of the
multivariate Student-$t$ distribution in $\mathbb{R}^{d}$ with $\nu$ degrees
of freedom and scale matrix $\Sigma$, and
\begin{equation}
\label{eq:student-mcp}
g(x):=\sum_{i=1}^{d}p_\lambda(x_i),\qquad
p_\lambda(x):=\begin{cases}\lambda|x|-\dfrac{x^{2}}{2\gamma}, & |x|\le\gamma\lambda,\\[6pt]
\dfrac{\gamma\lambda^{2}}{2}, & |x|>\gamma\lambda,\end{cases}
\end{equation}
is the minimax concave penalty (MCP) with regularization parameter $\lambda>0$
and scaling parameter $\gamma>1$. Thus $U=f+g$ with
$f=(\beta+1)\log q$ smooth and $g$ non-differentiable on the coordinate
hyperplanes $\{x_i=0\}$. We take $d=3$, $\nu=6$,
$\Sigma=\operatorname{diag}(10^{-2},10^{-1},1)$, so that the Student-$t$ part
is anisotropic with condition number $\kappa=100$, and $\lambda=1$,
$\gamma=2$. 
Since $0\le g\le d\gamma\lambda^{2}/2$ is bounded, the target keeps
the polynomial tails of the Student-$t$ distribution: only the moments of
order below $\nu$ are finite, and
$\nabla f(x)=\frac{2(\beta+1)}{\nu}\,\Sigma^{-1}x/q(x)$ vanishes at infinity,
which is the regime in which the overdamped Langevin dynamics
\eqref{eq:overdamped-2} loses its restoring force and is known to converge
slowly. At the same time $\nabla U$ does not exist on $\{x_i=0\}$, where
$g$ has its cusps, so \eqref{eq:overdamped-2} cannot be used without
smoothing. 
Following Section~\ref{sec:anchored-heavy-tail} and \eqref{eq:U0anchor} we
take
\begin{equation}
\label{eq:student-anchor}
U_0=\beta\log q+g_0,\qquad U=(\beta+1)\log q+g,\qquad
a=e^{U-U_0}=q\,e^{\,g-g_0},
\end{equation}
where $g_0(x)=\sum_{i=1}^{d}p^{\varepsilon}_\lambda(x_i)$ is the
deterministic smoothing of $g$ used in \citet{gurbuzbalaban2025anchored},
\begin{equation}
\label{eq:student-mcp-eps}
p^{\varepsilon}_\lambda(t):=\begin{cases}
\lambda\sqrt{t^{2}+\varepsilon^{2}}-\dfrac{\lambda\,t^{2}}{2\sqrt{\gamma^{2}\lambda^{2}+\varepsilon^{2}}},
& |t|\le\gamma\lambda,\\[8pt]
\dfrac{\lambda\,(\gamma^{2}\lambda^{2}+2\varepsilon^{2})}{2\sqrt{\gamma^{2}\lambda^{2}+\varepsilon^{2}}},
& |t|>\gamma\lambda,\end{cases}
\end{equation}
with $\varepsilon>0$ (we take $\varepsilon=0.1$). The function $p^{\varepsilon}_\lambda$ is continuously
differentiable, with
\begin{equation}
\label{eq:student-mcp-eps-grad}
(p^{\varepsilon}_\lambda)'(t)=\begin{cases}
\lambda t\left(\dfrac{1}{\sqrt{t^{2}+\varepsilon^{2}}}-\dfrac{1}{\sqrt{\gamma^{2}\lambda^{2}+\varepsilon^{2}}}\right),
& |t|\le\gamma\lambda,\\[8pt]
0, & |t|>\gamma\lambda,\end{cases}
\end{equation}
which is bounded by $\lambda$, Lipschitz (with constant at most
$\lambda/\varepsilon$), has the sign of $t$ and vanishes for
$|t|>\gamma\lambda$, and $p^{\varepsilon}_\lambda\to p_\lambda$
uniformly as $\varepsilon\to0$. A
direct computation shows that $t\mapsto p^{\varepsilon}_\lambda(t)-p_\lambda(t)$
is positive, even, and non-increasing in $|t|$ on $[0,\gamma\lambda]$, so that
$0<p^{\varepsilon}_\lambda-p_\lambda\le p^{\varepsilon}_\lambda(0)-p_\lambda(0)=\lambda\varepsilon$,
and therefore
\begin{equation}
\label{eq:student-clock}
e^{-d\lambda\varepsilon}\,q(x)\le a(x)\le q(x),
\end{equation}
so that the anchored clock is, up to a factor in $[0.74,1]$, the quadratic
clock $a=q$ of the family of Section~\ref{sec:anchored-heavy-tail} with
$r=0$ and $\zeta=1$ and a scale matrix $\Sigma$ in place of the identity.
Since $\nabla U_0=\frac{2\beta}{\nu}\Sigma^{-1}x/q+\nabla g_0$, for a constant
skew-symmetric matrix $J$ the NRALD \eqref{eq:student-zeta-sde} reads
\begin{equation}
\label{eq:student-NALD}
dX_t=-e^{\,g(X_t)-g_0(X_t)}\,(I_d+\alpha J)\left(\frac{2\beta}{\nu}\,\Sigma^{-1}X_t
+q(X_t)\,\nabla g_0(X_t)\right)\,dt+\sqrt{2\,a(X_t)}\;dW_t ,
\end{equation}
with $\partial_i g_0(x)=(p^{\varepsilon}_\lambda)'(x_i)$.  
In the absence of the regularizer
($\lambda=0$) the two factors of $q$ cancel and the drift is linear,
$-\frac{2\beta}{\nu}(I_d+\alpha J)\Sigma^{-1}X_t$; the regularizer adds the term
$q\,\nabla g_0$, whose $i$-th component has the sign of $x_i$ and vanishes
for $|x_i|>\gamma\lambda$. 
The heavy tails of $\pi$ are
produced entirely by the state-dependent clock in the diffusion
coefficient, and only the smooth anchor $U_0$ is differentiated; the
non-smooth target $U$ enters through function evaluations in $a$.
In our experiments, we take the constant skew-symmmetric matrix $J=J_{a}$
as:
\begin{equation}\label{eq:J_a}
    J_a=\begin{pmatrix}0&0&a\\
0&0&0\\
-a&0&0\end{pmatrix}.
\end{equation}

\begin{figure}[!htbp]
    \centering
    \includegraphics[width = \linewidth]{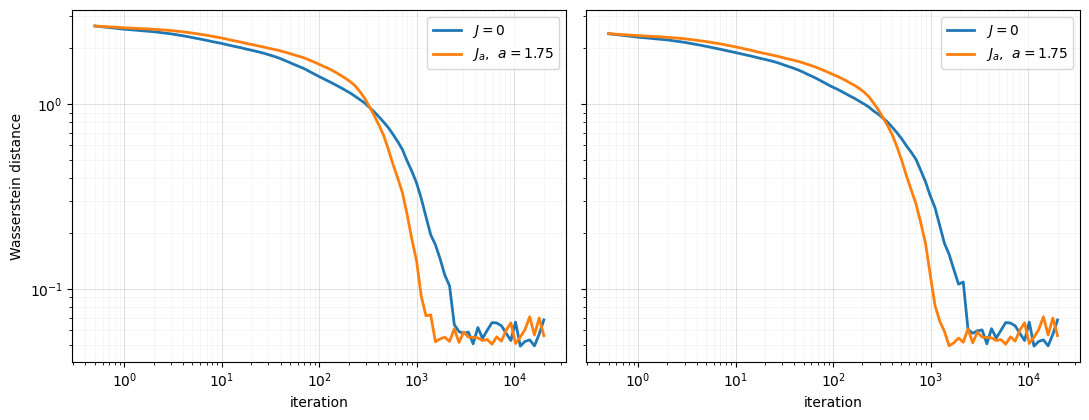}
    \caption{$2$-Wasserstein distance between the
    particle distribution and an exact sample of the target
    \eqref{eq:student-target} with the MCP regularizer \eqref{eq:student-mcp},
    as a function of the iterate on logarithmic axes, for the reversible
    anchored Langevin algorithm ($\alpha=0$) and NRALD with the constant
    skew-symmetric matrix $J_a$ ($\alpha=1$, $a=1.75$),
    averaged over $5$ replications of $N=5000$ particles, from the initial
    distributions $\mathcal N(0,10\,I_d)$ (left) and $\mathcal U(-5,5)^{d}$ (right).}
    \label{fig:heavy}
\end{figure}
Figure~\ref{fig:heavy} shows the two curves for the two initial
distributions. During the first three hundred iterates the non-reversible
chain is slightly behind, since its step size is smaller and the ramp
$\theta_k$ is still switching the skew term on. Once the skew term is
active, the non-reversible chain converges faster: starting from
$\mathcal N(0,10\,I_d)$, the $2$-Wasserstein distance falls below
$0.5$ after $568$ iterates for NRALD against $794$ for the reversible chain and at iterate $1000$ the distance is $0.12$ against
$0.37$. Starting from the uniform distribution the numbers are $508$
against $794$, and $0.13$
against $0.32$ at iterate $1000$, so the picture does not depend on the
initial distribution. The number of iterates needed to reach twice the
noise floor is therefore reduced by a factor of $1.95$ (Gaussian start) and
$1.74$ (uniform start), and the physical time $k\eta$ by a factor of $3.7$
and $3.3$ once the smaller step size of the non-reversible chain is taken
into account. Both chains then fluctuate at the same level,
$\mathrm{W}_2\approx0.057$, which is the noise level of the two-sample
estimator (its floor is $0.064$), in agreement with the fact that they
share the invariant distribution $\pi$. We stress that the gain is obtained
although the target is heavy-tailed, anisotropic and non-differentiable:
the anchored construction replaces the vanishing gradient of $f$ and the
non-existent gradient of $g$ by the gradient of the smooth anchor $U_0$,
whose leading part is linear, and the constant skew-symmetric matrix $J_a$ from ~\eqref{eq:J_a} then
accelerates this linear part in the same way as for a Gaussian target,
while the state-dependent clock $a$ restores the correct heavy-tailed,
non-smooth target.

\subsection{Constrained sampling}

\subsubsection{Truncated elliptical and \texorpdfstring{$\ell_{1}$}{l1}-Laplace distributions}\label{toy_example}

We test NRALD ~\eqref{eq:nrald} on two non-smooth targets restricted to a ball and $\ell_{1}$. Since exact samples of
these targets are available, so that the convergence of the algorithms can
be measured directly by the $\mathcal W_1$ distance to the target rather
than by a downstream accuracy. Let
\begin{equation}
\label{eq:ball-Sigma}
\Sigma:=DCD,\qquad D=\operatorname{diag}(0.5,\,1,\,2),\qquad
C_{ii}=1,\quad C_{ij}=0.6\ (i\neq j),
\end{equation}
so that $\Sigma$ has eigenvalues $0.131$, $0.595$ and $4.525$ , let $L$ be its lower-triangular Cholesky factor,
$\Sigma=LL^\top$, and write $y:=L^{-1}x$ for the whitened coordinates. The
two targets are $\pi(x)\propto e^{-U(x)}\mathbf 1_K(x)$ with
\begin{align}
&\text{Target (A)}\quad U(x)=\left\|L^{-1}x\right\|_2=\left(x^\top\Sigma^{-1}x\right)^{1/2},\label{target_a}
\\
&\text{Target (B)}\quad U(x)=\left\|L^{-1}x\right\|_1=\sum_{i=1}^{3}\left|(L^{-1}x)_i\right|,\label{target_b}
\end{align}
and
\begin{equation}
\label{eq:ball-K}
K=\left\{x\in\mathbb{R}^{3}:\|x\|_2\le R\right\},\qquad R=\frac12,
\end{equation}
i.e.\ the centered ball \eqref{eq:ball} with $r=R^{2}=\frac14$. The potential of
~\eqref{target_a} is Lipschitz and smooth except at the origin, where it has a conical
singularity, while that of ~\eqref{target_b} is Lipschitz and non-differentiable on the
three planes $\{(L^{-1}x)_i=0\}$ through the origin; in both cases
$\nabla U$ fails to exist exactly where the density is largest. The
constraint is severe: only $2.0\%$ of target (A) in \eqref{target_a} and $4.5\%$ of target (B) in \eqref{target_b} of the mass of the
unconstrained distributions lies in $K$, and the restricted targets are close
to the uniform distribution on $K$. Whatever acceleration is observed below in Figure \ref{fig:toy} is therefore due to
the interaction of the circulation with the constraint and with the
non-smooth core, not to an elongated target.

Following \cite{gurbuzbalaban2025anchored} we take as anchor the
deterministic smoothing of $U$ in the whitened coordinates,
\begin{equation}
\label{eq:ball-anchor}
\text{(A)}\quad U_0(x)=\sqrt{\left\|L^{-1}x\right\|_2^{2}+\delta^{2}},
\qquad\qquad
\text{(B)}\quad U_0(x)=\sum_{i=1}^{3}\sqrt{\left(L^{-1}x\right)_i^{2}+\delta^{2}},
\end{equation}
with $\delta=0.1$, which is $C^\infty$ with
\begin{equation}
\label{eq:ball-gradU0}
\text{(A)}\quad \nabla U_0(x)=\frac{L^{-\top}y}{\sqrt{\|y\|_2^{2}+\delta^{2}}},
\qquad\qquad
\text{(B)}\quad \nabla U_0(x)=L^{-\top}\left(\frac{y_i}{\sqrt{y_i^{2}+\delta^{2}}}\right)_{i=1}^{3},
\end{equation}
so that $\|\nabla U_0\|_2\le\|L^{-1}\|_2=2.77$ in case (A) and
$\|\nabla U_0\|_2\le\sqrt3\,\|L^{-1}\|_2=4.79$ in case (B). Since
$0<\sqrt{t^{2}+\delta^{2}}-|t|\le\delta$, the clock satisfies
\begin{equation}
\label{eq:ball-clock}
a(x)=e^{U(x)-U_0(x)}\in\big[e^{-\delta},1\big)=[0.905,1)\ \ \text{(A)},
\qquad
a(x)\in\big[e^{-3\delta},1\big)=[0.741,1)\ \ \text{(B)},
\end{equation}
and it is Lipschitz on $K$ because $U$ is Lipschitz and $U_0$ is smooth.
Only $U_0$ is differentiated; the non-smooth $U$ enters through function
evaluations in $a$.

We run the projected scheme \eqref{eq:algorithm-constrained} with the state-dependent skew-symmetric matrix $J=J_s$ from ~\eqref{eq:Js}, the
projection $\Pi_K(z)=z\min\{1,R/\|z\|_2\}$ and, for all schemes, the same
step size $\eta=5\times10^{-4}$ and the same number of iterates, $5000$, so
that the reversible and non-reversible chains use the same number of
evaluations of $U$ and $\nabla U_0$ and differ only by one cross product per
iterate. We propagate $N=5000$ independent walkers, all started at the
boundary point $x_0=\Pi_K(0.5,0.5,0.5)=(1,1,1)/(2\sqrt3)$, and repeat the
experiment over $8$ independent replicas. Exact samples of $\pi$ are
obtained by rejection: we draw from the unconstrained law of
~\eqref{target_a} and ~\eqref{target_b} and keep the draws that fall in $K$. As a measure of
convergence we use, every $25$ iterates, the $\mathcal W_1$ distance
between the empirical distribution of the walkers and that of $5000$ exact
samples in each coordinate, averaged over the three coordinates,
\begin{equation}
\label{eq:ball-W1}
\mathcal W_1(k):=\frac13\sum_{i=1}^{3}\mathcal W_1\big(\hat\mu_{k,i},\hat\pi_i\big),
\end{equation}
and we report its mean over the replicas. Being a
two-sample quantity, \eqref{eq:ball-W1} does not vanish for a perfect
sampler: between two independent exact samples of size $5000$ it equals
$0.0050$ for ~\eqref{target_a} and $0.0043$ for ~\eqref{target_b} on average, which we call the floor. For each
replica and each $m\in\{3,5\}$ we record the number of iterates $\tau_m$
after which $\mathcal W_1(k)$ stays below $m$ times the floor, and we report
the mean of $\tau_m$ over the replicas with a bootstrap $95\%$ interval, as
well as the ratio of the mean $\tau_m$ of the reversible chain to that of
the non-reversible one.

\begin{figure}[htbp]
    \centering
    \includegraphics[width=\linewidth]{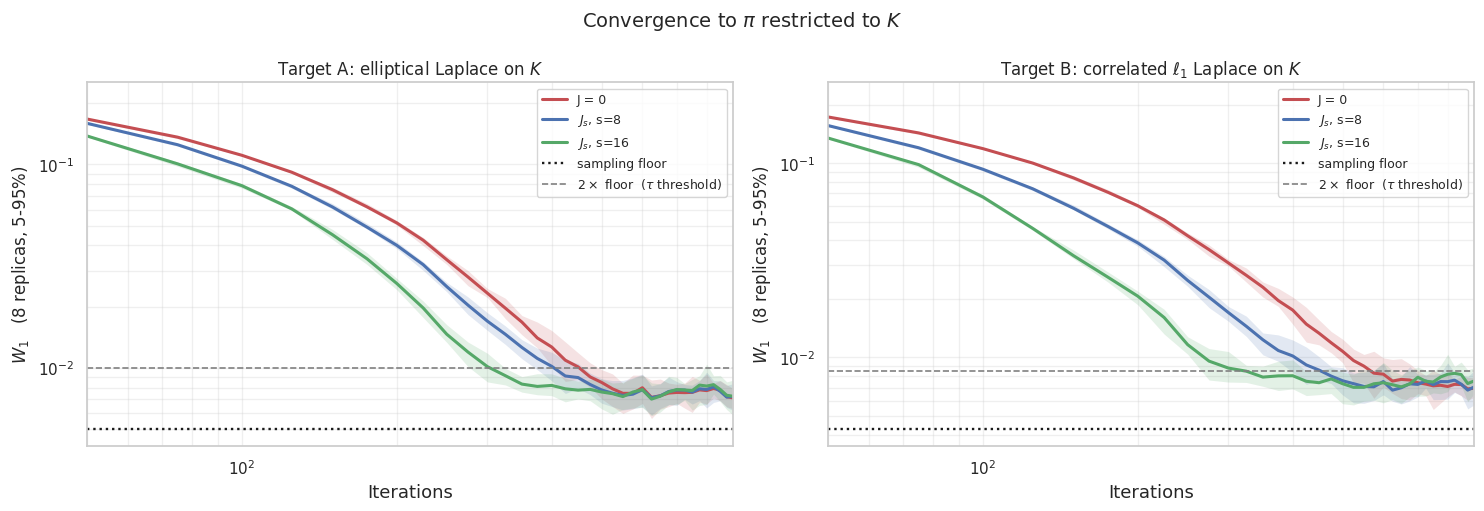}
    \caption{$\mathcal W_1$ distance \eqref{eq:ball-W1} between the walkers
    and an exact sample of the target, averaged over $8$ replicas, as a
    function of the iterate on logarithmic axes, for the reversible
    projected anchored Langevin algorithm ($\alpha=0$) and NRALD with the
    tensor $J_s$ in \eqref{eq:Js} with $s=8$ and $s=16$ ($\alpha=1$), for the
    targets (A) ~\eqref{target_a} (left)  and (B) ~\eqref{target_b} (right) on the ball
    \eqref{eq:ball-K}. The dotted line is the two-sample floor and the
    dashed line is $3\times$floor, the threshold of $\tau_3$.}
    \label{fig:toy}
\end{figure}

\begin{figure}[!htbp]
    \centering
    \includegraphics[width=\linewidth]{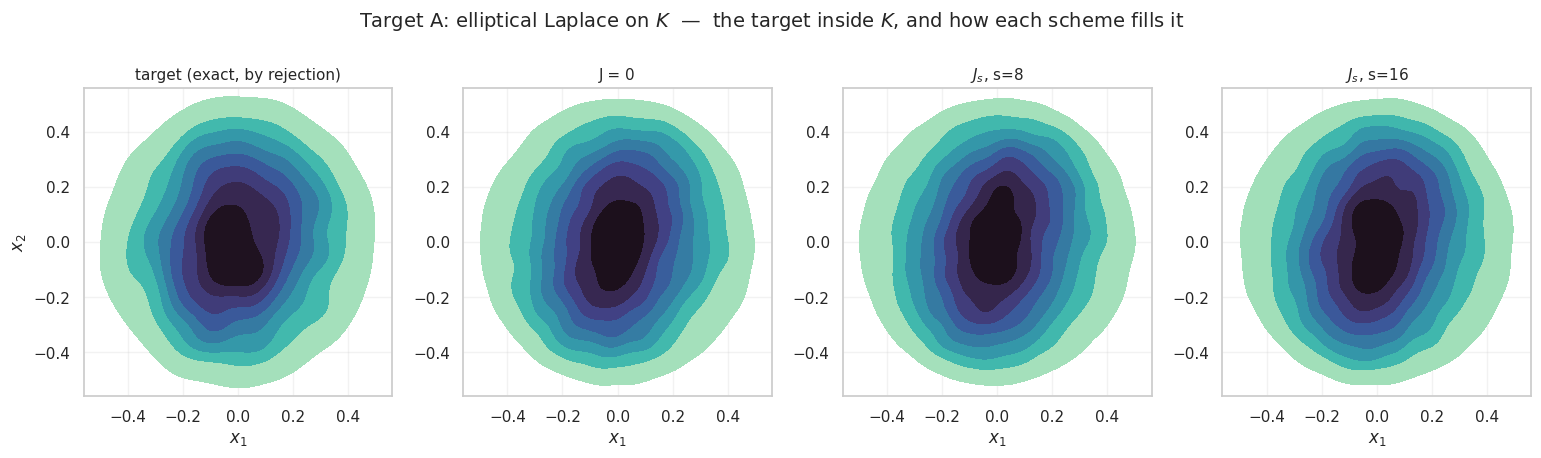}
    \caption{Visualized density plots of the first two coordinates $(x_1,x_2)$
    for target (A), $U(x)=\|L^{-1}x\|_2$ in \eqref{target_a}, restricted to
    the ball $K$ in \eqref{eq:ball-K}. From left to right: $5000$ exact samples
    of the target obtained by rejection, and the walkers of the reversible
    projected anchored Langevin algorithm ($J=0$) and of NRALD with $J_s$ in
    \eqref{eq:Js} for $s=8$ and $s=16$ after $5000$ iterates with
    $\eta=5\times10^{-4}$ ($8$ replicas of $5000$ walkers), shown as kernel
    density estimates with the same contour levels.}
    \label{fig:elliptic}
\end{figure}

\begin{figure}[!htbp]
    \centering
    \includegraphics[width=\linewidth]{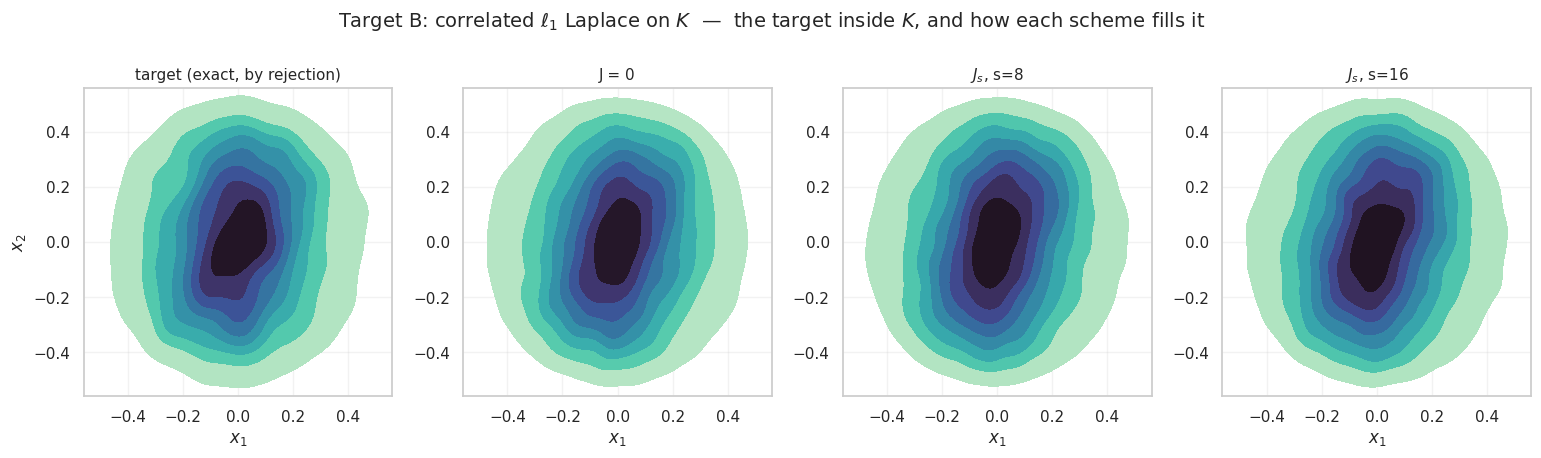}
    \caption{Visualized density plots of the first two coordinates $(x_1,x_2)$
    for target (B), $U(x)=\|L^{-1}x\|_1$ in \eqref{target_b}, restricted to
    the ball $K$ in \eqref{eq:ball-K}. From left to right: $5000$ exact samples
    of the target obtained by rejection, and the walkers of the reversible
    projected anchored Langevin algorithm ($J=0$) and of NRALD with $J_s$ in
    \eqref{eq:Js} for $s=8$ and $s=16$ after $5000$ iterates with
    $\eta=5\times10^{-4}$ ($8$ replicas of $5000$ walkers), shown as kernel
    density estimates with the same contour levels.}
    \label{fig:l1}
\end{figure}

Figure~\ref{fig:toy} shows that NRALD with
$J_s$ \eqref{eq:Js} converges faster than the reversible projected anchored Langevin
algorithm on both targets, and that the gain grows with $s$. On target (A) ~\eqref{target_a},
$\mathcal W_1$ after $100$, $200$ and $300$ iterates is $0.111$, $0.052$ and
$0.023$ for the reversible chain, $0.098$, $0.040$ and $0.017$ for $s=8$, and
$0.079$, $0.026$ and $0.010$ for $s=16$; on target (B) the corresponding
values are $0.119$, $0.060$, $0.031$, then $0.093$, $0.039$, $0.017$, and
$0.067$, $0.021$, $0.009$. The number of iterates needed to reach
$3\times$floor drops from $381$ to $338$ ($s=8$) and $259$ ($s=16$) on (A),
and from $469$ to $369$ and $253$ on (B). Doubling $s$ roughly doubles the gain, and the gain is
larger for target (B), whose potential is non-smooth on three planes, than
for target (A), whose potential is non-smooth only at the origin. Once converged, all schemes settle at the same stationary level. The
stationary $\mathcal W_1$ is $0.0075$ for the three schemes on (A) and
$0.0070$, $0.0069$ and $0.0074$ on (B), within one bootstrap interval of
each other. Figures~\ref{fig:elliptic}
and~\ref{fig:l1} confirm this visually: after $5000$ iterates the
$(x_1,x_2)$ marginals of the three schemes are indistinguishable from that of
the exact sample, including the tilted core of target (B). This is what the
theory predicts. The tensor $J_s$ preserves $\pi$, and since it is
tangential it adds no normal component to the step near $\partial K$, so it
leaves the projection error of the scheme unchanged.

\subsubsection{Bayesian logistic regression: Synthetic data}

We test NRALD ~\eqref{eq:nrald} on a binary classification synthetic problem whose target is not
differentiable. Suppose we can access a dataset $Z=\{z_j\}_{j=1}^{n}$,
$z_j=(X_j,y_j)$, where $X_j\in\mathbb{R}^{d}$ are the features and
$y_j\in\{0,1\}$ the labels, and assume that, conditionally on the features
and on the regression coefficients $\beta\in\mathbb{R}^{d}$, the labels are
independent with
\begin{equation}
\label{eq:logit}
\mathbb{P}\left(y_j=1\mid X_j,\beta\right)=\frac{1}{1+e^{-\beta^\top X_j}} .
\end{equation}
The first coordinate of every $X_j$ equals $1$, so $\beta_1$ is an
intercept, and we write $\beta=(\beta_1,\beta_{2:d})$. The prior on $\beta$
is the product of a Gaussian $\mathcal N(0,\sigma^2)$ on the intercept, a
Laplace (LASSO) prior with parameter $\lambda>0$ on the slopes, and the
indicator of a compact set $K\subset\mathbb{R}^d$. The target is therefore
$\pi(\beta)\propto e^{-U(\beta)}\mathbf 1_K(\beta)$ with $U=f+g$ of the form
\eqref{eq:Ufg}, where
\begin{equation}
\label{eq:logistic-fg}
f(\beta):=\sum_{j=1}^{n}\left[\log\left(1+e^{\beta^\top X_j}\right)-y_j\,\beta^\top X_j\right]
+\frac{\beta_1^{2}}{2\sigma^{2}},
\qquad
g(\beta):=\lambda\sum_{i=2}^{d}|\beta_i| .
\end{equation}
Here $f$ is convex and $C^\infty$, whereas $g$ is the LASSO regularizer restricted to the slopes: it is convex and
Lipschitz but not differentiable on the coordinate hyperplanes
$\{\beta_i=0\}$, $i=2,\dots,d$, which is precisely where a sparse posterior
puts its mass. Hence $U$ is locally Lipschitz but $\nabla U$ does not exist
on a set of positive $\pi$-probability (near which the chain spends most of
its time), and the reflected Langevin dynamics \eqref{eqn:RLD} cannot be
used. Following \eqref{eq:U0anchor}, we take the anchor $U_0=f+g_0$ with
\begin{equation}
\label{eq:logistic-anchor}
g_0(\beta):=\lambda\sum_{i=2}^{d}\sqrt{\beta_i^{2}+\delta^{2}},
\qquad
a(\beta)=e^{U(\beta)-U_0(\beta)}=e^{\,g(\beta)-g_0(\beta)}
\in\left[e^{-(d-1)\lambda\delta},\,1\right],
\end{equation}
for a small $\delta>0$.
Only $U_0$ is differentiated; the target $U$ enters through function
evaluations in the clock $a$.
Consider the following example of $d=9$, where the first coordinate of $X_j$
is the intercept. We first generate $n=2000$ synthetic data by the following
model
\begin{equation}
\label{eq:synthetic}
X_j=\left(1,\widetilde X_j\right),\quad
\widetilde X_j\sim\mathcal{N}\left(0,\Sigma\right),\quad
p_j\sim\mathcal{U}(0,1),\quad
y_j=\begin{cases}1&\text{if } p_j\le\dfrac{1}{1+e^{-\beta_\ast^\top X_j}}\\[6pt] 0&\text{otherwise}\end{cases},
\end{equation}
where $\mathcal{U}(0,1)$ stands for the uniform distribution on $[0,1]$,
$\beta_\ast\in\mathbb{R}^{9}$ is the true coefficient vector, and
$\Sigma\in\mathbb{R}^{8\times8}$ is block diagonal. In
\eqref{eq:logistic-fg}--\eqref{eq:logistic-anchor} we take $\lambda=2$,
$\sigma=5$ and $\delta=0.02$, and $\nabla U_0$ is computed over the full
training set. Then we implement the scheme \eqref{eq:algorithm-constrained} with
$1000$ iterates, initialized from the uniform distribution on $K$, and repeat
it over $R=100$ independent replicates, where the reversible ($\alpha=0$) and
non-reversible ($\alpha=1$) chains of a replicate share the same initial
point and the same Gaussian increments. For the constrained Bayesian
logistic regression, it is not practical for us to compute the
$\mathcal{W}_1$ distance between the approximated Gibbs distribution and the
empirical distribution. Hence, for such a binary classification problem, we
use the accuracy of the current iterate $\beta_k$ over the training set and
the test set to measure the goodness of the convergence performance.
 
In the experiment with the ball constraint set in \eqref{eq:ball} with
$r=1$, we choose the skew-symmetric matrix $J_s(x)$ in \eqref{eq:Js} with
$s=16$, ramped up linearly from $0$ over the first $20$ iterates, and the step
size $\eta=3\times10^{-6}$, for the non-reversible anchored Langevin
algorithm shown by the red solid line. We use $20\%$ of the whole dataset as
the test set. The mean and standard deviation of the accuracy distribution
over the $R=100$ replicates are shown in Figure~\ref{fig:logsyth-ball}.

\begin{figure}[!htbp]
    \centering
    \includegraphics[width=\linewidth]{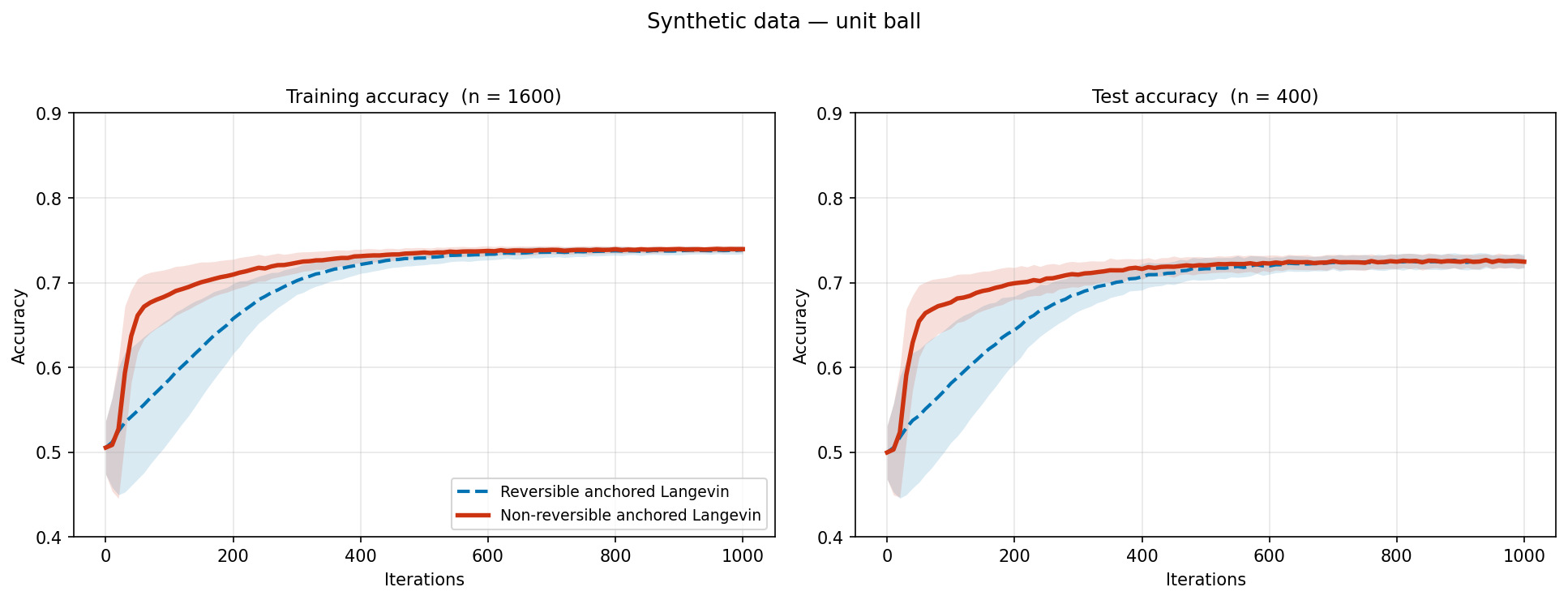}
       
    \caption{Accuracy over the training set and the test set for the
    synthetic data. The red part (solid line) denotes the mean and standard
    deviation of the non-reversible anchored Langevin algorithm with $J_s(x)$
    ($\alpha=1$) and the blue part (dashed line) denotes the mean and standard
    deviation of the reversible anchored Langevin algorithm ($\alpha=0$).}
    \label{fig:logsyth-ball}
\end{figure}
 
Under the same setting, we consider the smoothed $\ell_1$ ball constraint in
\eqref{eq:l1ball} with $\varepsilon=0.2$ and $\Lambda=7.4$ in this synthetic
data example, the scaling parameter $s=4$ for $J_g(x)$ defined in
\eqref{eq:Jg}, and the step size $\eta=7\times10^{-6}$. We obtain
Figure~\ref{fig:logsyth-l1}.
 
\begin{figure}[!htbp]
   \centering
   \includegraphics[width=\linewidth]{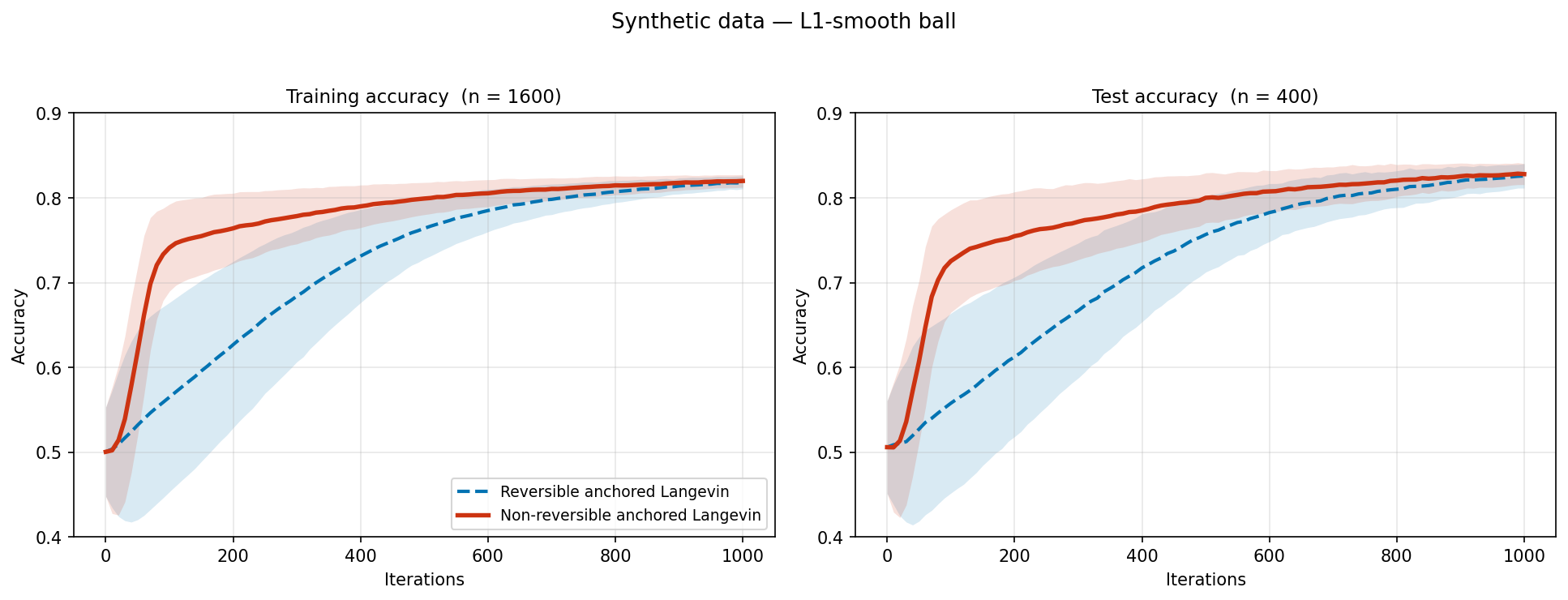}
    \caption{Accuracy over the training set and the test set for the
    synthetic data. The red part (solid line) denotes the mean and standard
    deviation of the non-reversible anchored Langevin algorithm with $J_g(x)$
    ($\alpha=1$) and the blue part (dashed line) denotes the mean and standard
    deviation of the reversible anchored Langevin algorithm ($\alpha=0$).}
    \label{fig:logsyth-l1}
\end{figure}
 
We observe that the non-reversible anchored Langevin algorithm, with
$J_s(x)$ in the red line in Figure~\ref{fig:logsyth-ball} and with $J_g(x)$ in
the red line in Figure~\ref{fig:logsyth-l1}, requires fewer iterations to
achieve the same level of accuracy on the test set and exhibits more stable
sampling behavior comparing to the reversible anchored Langevin algorithm in
the blue dashed line. In particular, after $100$ iterates the mean test
accuracy is $0.72$ versus $0.56$ for the smoothed $\ell_1$ ball constraint and
$0.68$ versus $0.58$ for the ball constraint, while both algorithms reach the
same accuracy level by $1000$ iterates, in agreement with the fact that they
share the same invariant distribution.

\subsubsection{Bayesian logistic regression: Real data}

In this section, we consider the constrained Bayesian logistic regression
problem on the MAGIC Gamma Telescope dataset\footnote{The Telescope dataset
is publicly available from
\url{https://archive.ics.uci.edu/ml/datasets/magic+gamma+telescope}.} and
the Titanic dataset\footnote{The Titanic dataset is publicly available from
\url{https://www.kaggle.com/c/titanic}.}. The Telescope dataset contains
$n=19020$ samples with $10$ features, describing the registration of high
energy gamma particles in a ground-based atmospheric Cherenkov gamma
telescope using the imaging technique, and the goal is to classify each
event as a gamma or a hadron shower. We append one identically zero feature
so that, with the intercept, the dimension $d=12$ is a multiple of $3$ and
matches the block structure of $J_s(x)$ and $J_g(x)$; the coefficient of this
feature is pulled to zero by the LASSO term and plays no role in the
predictions. The Titanic dataset contains $891$ samples with $11$ attributes
describing the passengers, and the goal is to predict whether a passenger
survived or not. Some of the attributes are irrelevant to our goal or
mostly missing, namely the passenger identifier, the name, the ticket
number and the cabin, and we discard them. We keep the remaining seven
attributes, namely the passenger class, the sex, the age (with the $177$
missing values replaced by the median age), the number of siblings and
spouses aboard, the number of parents and children aboard, the fare (as
$\log(1+\text{fare})$) and the port of embarkation; the last one is
categorical with three levels and is encoded by two indicator variables.
We thus obtain a dataset containing $n=891$ labeled samples with $8$
feature columns, so that $d=9$ with the intercept. For both datasets we use $20\%$ of the whole dataset as the test
set, stratified by label, and we standardize the features by the mean and
standard deviation of the training set.

In \eqref{eq:logistic-fg}--\eqref{eq:logistic-anchor} we take $\lambda=2$,
$\sigma=5$ and $\delta=0.02$ as before, and $\nabla U_0$ is computed over the
full training set. For the ball constraint set $K_{r=1}$ in \eqref{eq:ball},
all features are rescaled by a common factor so that the pilot coefficient
vector has Euclidean norm $0.8$ and lies inside $K_{r=1}$; for the smoothed
$\ell_1$ ball constraint set in \eqref{eq:l1ball} we take $\varepsilon=0.2$
and $\Lambda=d\varepsilon+\|\hat\beta\|_1+1$, where $\hat\beta$ is the pilot
coefficient vector in the transformed coordinates. For both datasets, we
initialize the reversible ($\alpha=0$) and non-reversible ($\alpha=1$)
anchored Langevin algorithms with the uniform distribution on $K$, run
$R=100$ independent replicates in which the two chains of a replicate share
the same initial point and the same Gaussian increments, and report the
mean and standard deviation over the replicates of the accuracy of the
current iterate $\beta_k$ over the training set and the test set.

In the experiment of using the Telescope dataset with a centered ball
constraint set, we take the constraint set as the centered ball
$K_{r=1}\subset\mathbb{R}^{12}$ defined in \eqref{eq:ball}, and we set the step
size $\eta=2\times10^{-7}$ and implement the two algorithms for $1500$
iterates, where we choose the skew-symmetric matrix as $J_s(x)$ defined in
\eqref{eq:Js} with $s=16$ for the non-reversible anchored Langevin algorithm
shown by the red solid line. Figure~\ref{fig:magic-ball} reports the
accuracy of the two algorithms over the training set and the test set. In
the experiment of using the Telescope dataset with the smoothed $\ell_1$
ball constraint set defined in \eqref{eq:l1ball}, we choose the
skew-symmetric matrix as $J_g(x)$ defined in \eqref{eq:Jg} with $s=8$, the
step size $\eta=2\times10^{-6}$, and implement the two algorithms for $1000$
iterates to obtain Figure~\ref{fig:magic-l1}.

\begin{figure}[!htbp]
        \centering
        \includegraphics[width=\linewidth]{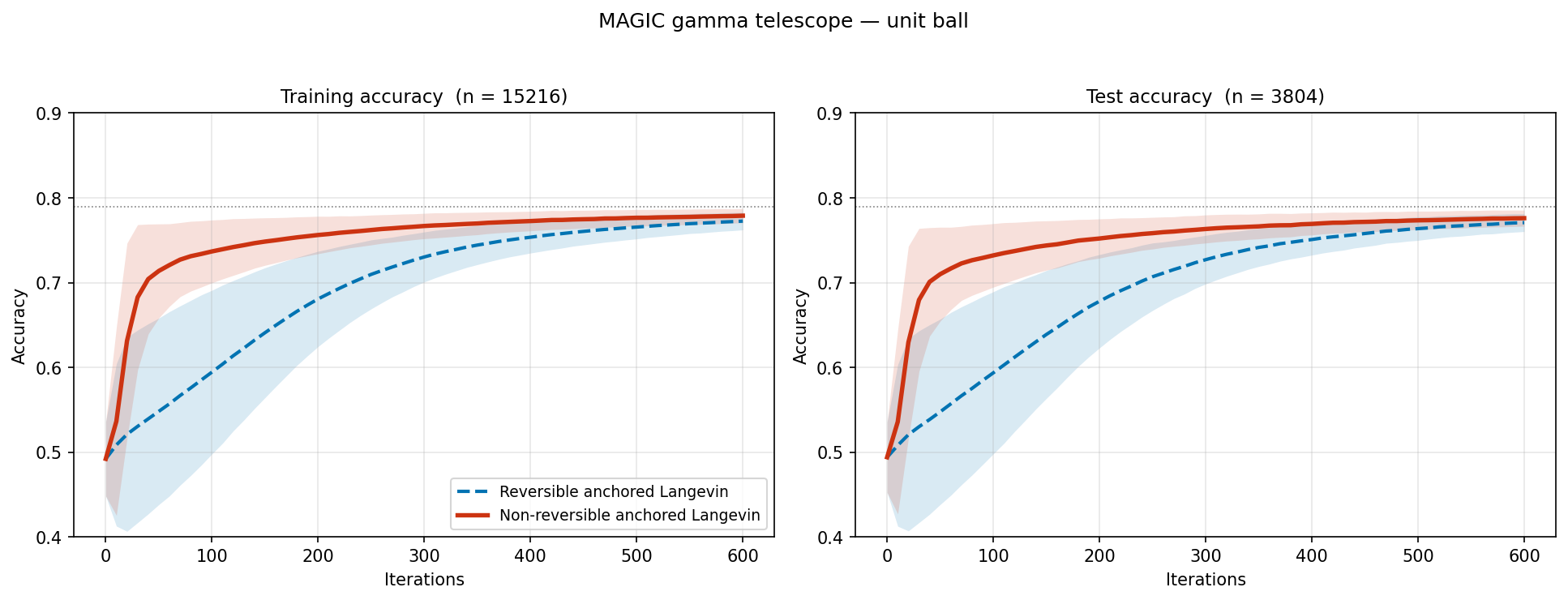}
    \caption{Accuracy over the training set and the test set for the
    Telescope dataset over the unit ball. The red part (solid line) denotes the mean and
    standard deviation of the non-reversible anchored Langevin algorithm with
    $J_s(x)$ ($\alpha=1$), the blue part (dashed line) denotes the mean and
    standard deviation of the reversible anchored Langevin algorithm
    ($\alpha=0$), and the grey dotted line is the test accuracy of the
    Bayesian logistic regression fit.}
    \label{fig:magic-ball}
\end{figure}

\begin{figure}[!htbp]
    \centering
    \includegraphics[width=\linewidth]{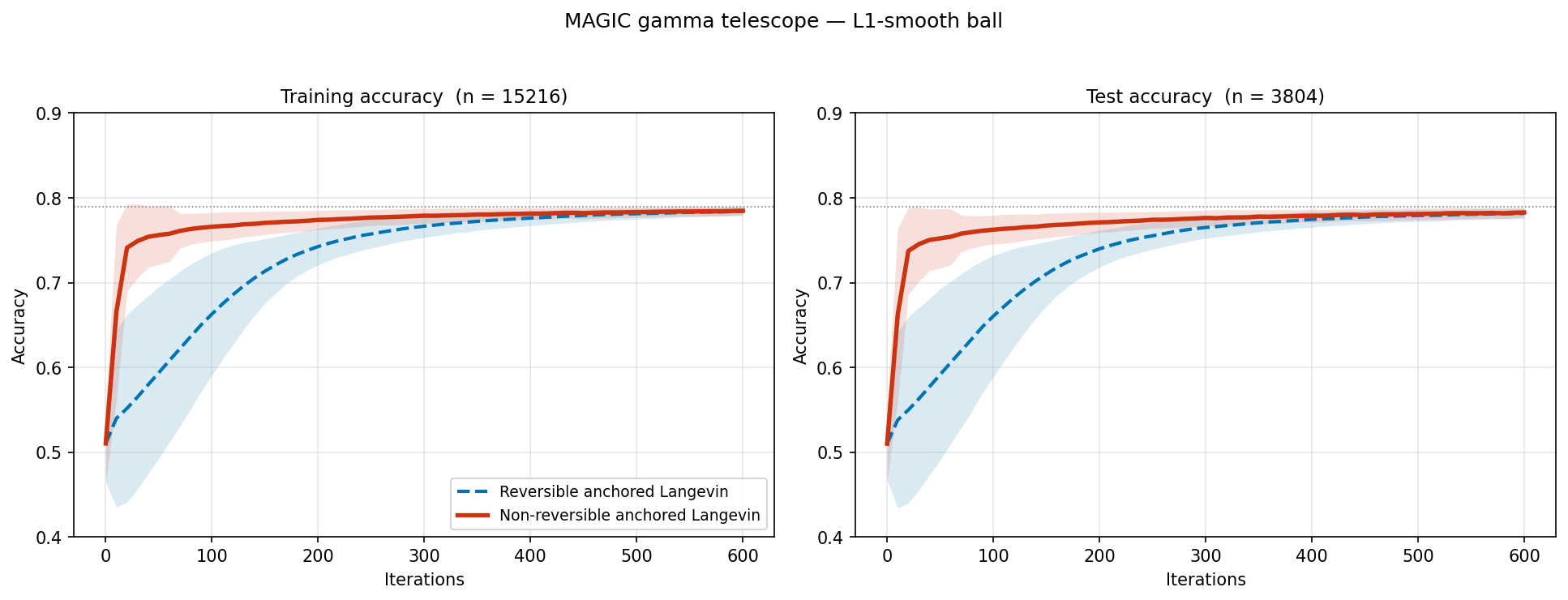}
    \caption{Accuracy over the training set and the test set for the
    Telescope dataset over the smoothed $l_1$ ball. The red part (solid line) denotes the mean and
    standard deviation of the non-reversible anchored Langevin algorithm with
    $J_g(x)$ ($\alpha=1$), the blue part (dashed line) denotes the mean and
    standard deviation of the reversible anchored Langevin algorithm
    ($\alpha=0$), and the grey dotted line is the test accuracy of the
    Bayesian logistic regression fit.}
    \label{fig:magic-l1}
\end{figure}

For the Titanic data, we first consider the constraint set as the centered
ball $K_{r=1}\subset\mathbb{R}^{9}$, and we set the step size
$\eta=1.4\times10^{-6}$ and implement the two algorithms for $2000$ iterates,
where we choose the skew-symmetric matrix as $J_s(x)$ defined in
\eqref{eq:Js} with $s=12$, ramped up linearly from $0$ over the first $10$
iterates, for the non-reversible anchored Langevin algorithm shown by the
red solid line. Figure~\ref{fig:titanic-ball} reports the accuracy
level of the two algorithms over the training and test sets. Taking the
constraint set as the smoothed $\ell_1$ ball constraint set defined in
\eqref{eq:l1ball}, we choose the skew-symmetric matrix as $J_g(x)$ defined in
\eqref{eq:Jg} with $s=8$, the step size $\eta=7\times10^{-6}$, and implement
the two algorithms for $3000$ iterates to obtain
Figure~\ref{fig:titanic-l1}.

\begin{figure}[htbp]
    \centering
    \includegraphics[width=\linewidth]{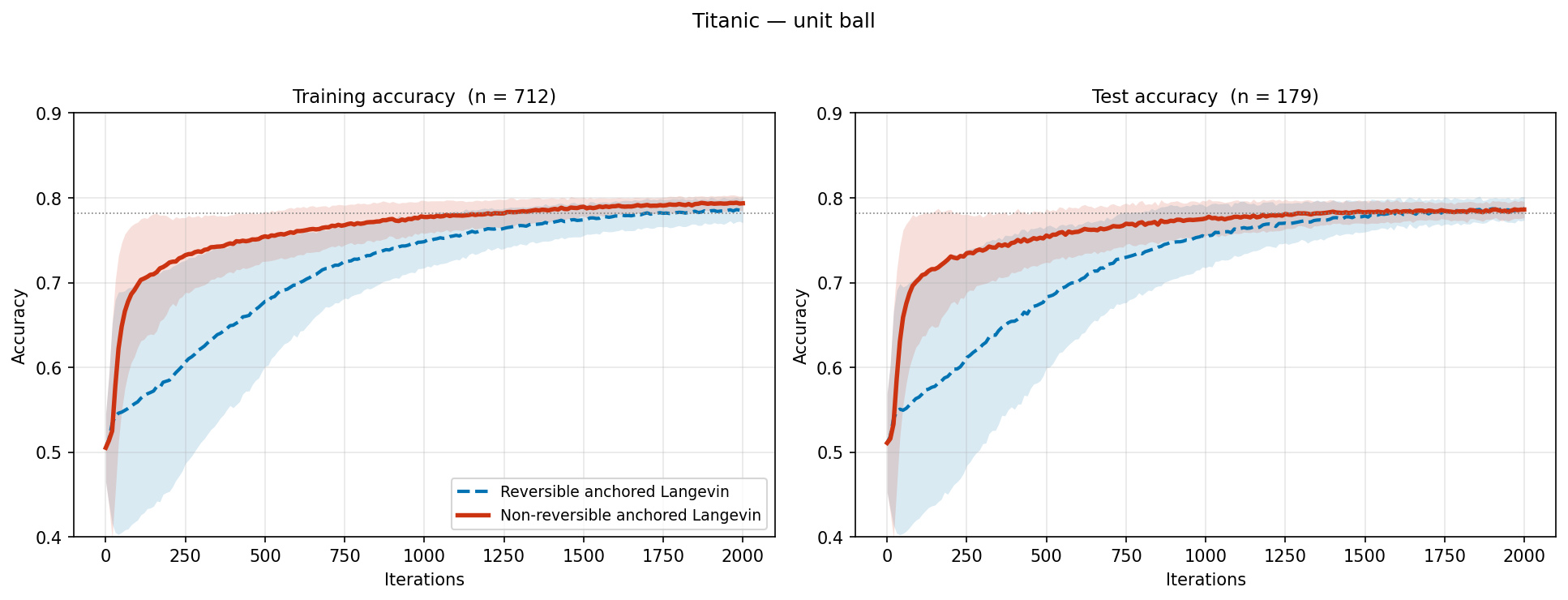}
    \caption{Accuracy over the training set and the test set for the Titanic
    dataset over the unit ball. The red part (solid line) denotes the mean and standard
    deviation of the non-reversible anchored Langevin algorithm with $J_s(x)$
    ($\alpha=1$), the blue part (dashed line) denotes the mean and standard
    deviation of the reversible anchored Langevin algorithm ($\alpha=0$), and
    the grey dotted line is the test accuracy of the Bayesian logistic regression
    fit.}
    \label{fig:titanic-ball}
\end{figure}

\begin{figure}[htbp]
    \centering
    \includegraphics[width=\linewidth]{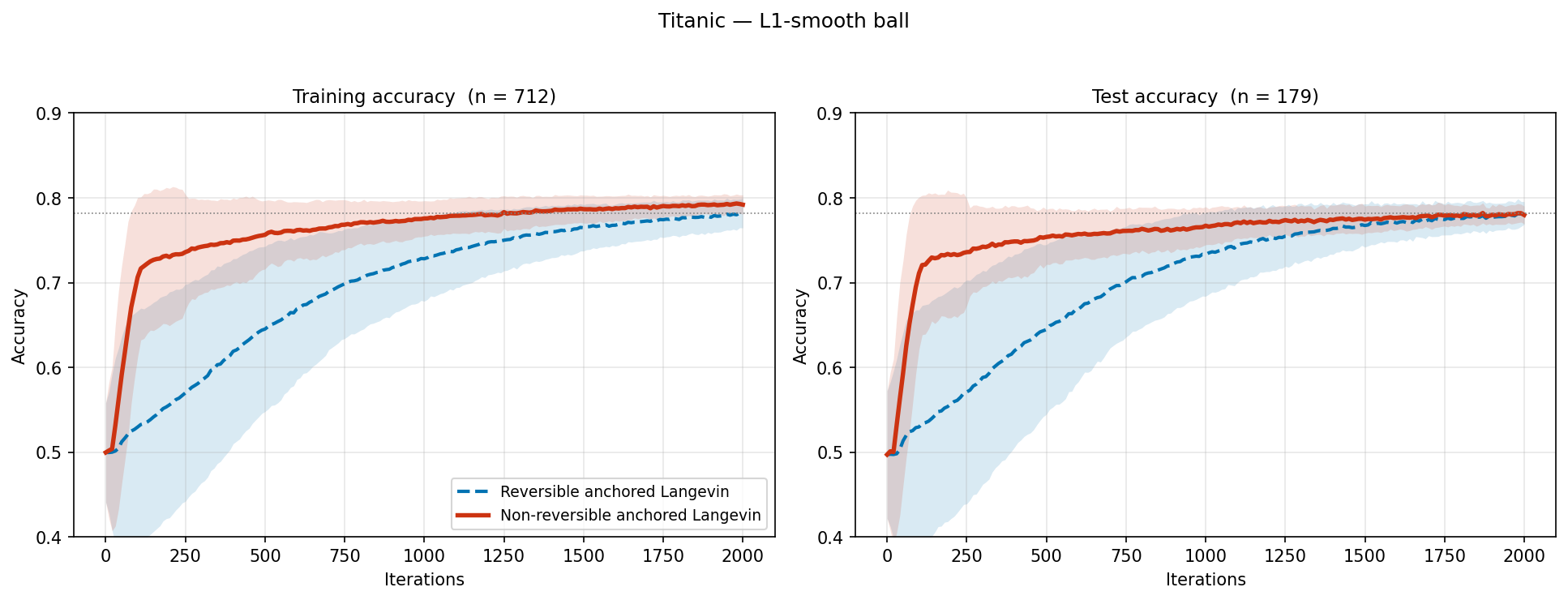}
    \caption{Accuracy over the training set and the test set for the Titanic
    dataset over smoothed $l_1$ ball. The red part (solid line) denotes the mean and standard
    deviation of the non-reversible anchored Langevin algorithm with $J_g(x)$
    ($\alpha=1$), the blue part (dashed line) denotes the mean and standard
    deviation of the reversible anchored Langevin algorithm ($\alpha=0$), and
    the grey dotted line is the test accuracy of the Bayesian logistic regression
    fit.}
    \label{fig:titanic-l1}
\end{figure}

From Figure~\ref{fig:magic-ball} and Figure~\ref{fig:magic-l1}
for the Telescope dataset, and Figure~\ref{fig:titanic-ball} and
Figure~\ref{fig:titanic-l1} for the Titanic dataset, we can conclude
that the non-reversible anchored Langevin algorithm ~\eqref{eq:nrald}, with either $J_s(x)$ or
$J_g(x)$, consistently attains a given level of classification accuracy with
fewer iterates than the reversible anchored Langevin algorithm ~\eqref{eqn:RALD} on both
datasets, which indicates a faster convergence rate. On the Telescope
dataset, the mean test accuracy after $50$ iterates is $0.71$ versus $0.55$
with the ball constraint and $0.75$ versus $0.59$ with the smoothed $\ell_1$
ball constraint; on the Titanic dataset, after $150$ iterates it is $0.72$
versus $0.58$ with the ball constraint and $0.73$ versus $0.54$ with the
smoothed $\ell_1$ ball constraint. The non-reversible algorithm also exhibits
a much smaller spread across replicates during the transient. By the end of
each run both algorithms have reached the accuracy of the Bayesian logistic
regression fit and agree with each other to within $0.005$, in agreement
with the fact that they share the same invariant distribution.

\section{Conclusion}\label{sec:discussion}

In this paper, we introduced non-reversible anchored Langevin dynamics (NALD) on
the whole Euclidean space and its normally reflected counterpart, non-reversible reflected anchored Langevin dynamics (NRALD),  on a bounded domain.  Both
dynamics combine a differentiable reference potential with a state-dependent
clock while allowing the target potential itself to be non-differentiable.
We constructed the nonreversible perturbation through a divergence-free
weighted current, proved invariance of the target distribution, and identified
the dynamics as random time changes of nonreversible reference diffusions.
We then quantified the resulting acceleration at three complementary levels.
We established finite-time total-variation convergence bounds, proved
level-two empirical-measure large deviation principles and an explicit
nonnegative decomposition of the irreversible rate function, and showed that
the asymptotic variance of stationary ergodic averages decreases under
the anti-symmetric perturbation.  The assumptions required on unbounded state
spaces were verified for a class of heavy-tailed anchored diffusions, including
a superquadratic Student-$t$ model, and an exactly solvable quadratic-clock
example illustrated the variance reduction for linear observables.
Numerical experiments are provided to illustrate the performance
of our proposed algorithms.

It will be interesting to obtain non-asymptotic convergence guarantees 
for the algorithms based on the discretization of the continuous-time 
NALD for unconstrained sampling
and NRALD for constrained sampling in the paper. 
The analysis can be challenging due to the state-dependent diffusion coefficient in NALD and NRALD
as well as the presence of reflections in NRALD. 
This will be left for future investigations.


\bibliographystyle{plainnat}
\bibliography{langevin}


\end{document}